\documentclass[11pt]{amsart}

\usepackage{graphicx}
\usepackage{siunitx}
\usepackage{amsfonts, amsthm}
\usepackage{graphicx}
\usepackage{xcolor}
\usepackage{colortbl}
\usepackage{algorithm}
\usepackage{algorithmic}
\usepackage{aliascnt}
\usepackage[inline]{enumitem}
\usepackage{multirow}
\usepackage[margin=1in]{geometry}
\usepackage{booktabs}
\setlist[enumerate]{leftmargin=.5in}
\setlist[itemize]{leftmargin=.5in}

\definecolor{p1color}{RGB} {200,  16,  46} 
\definecolor{p2color}{RGB} {  0, 179, 136} 
\definecolor{p3color}{RGB} {246, 190,   0} 
\definecolor{p4color}{RGB} {136, 139, 141} 
\definecolor{p5color}{RGB} {255, 249, 217} 
\definecolor{p6color}{RGB} { 51, 102, 155} 
\usepackage[colorlinks=true,
linkcolor=p1color,
citecolor=p1color]{hyperref}
\usepackage{cleveref}

\newtheorem{theorem}{Theorem}
\newtheorem{remark}{Remark}
\newtheorem{lemma}{Lemma}
\newtheorem{definition}{Definition}
\newtheorem{proposition}{Proposition}

\graphicspath{{figures}}

\crefname{remark}{remark}{remarks}
\Crefname{remark}{Remark}{Remarks}

\newcolumntype{C}{>{\columncolor{gray!20}}c}

\def\M{\mathcal{M}}
\def\D{\mathcal{D}}
\def\P{\mathcal{P}}

\def\R{\mathbb{R}}
\def\N{\mathbb{N}}
\newcommand{\Id}{\text{Id}}

\DeclareMathOperator*{\argmin}{arg\,min}

\DeclareMathOperator{\Tr}{Tr}

\newcommand{\defeq}{\ensuremath{\mathrel{\mathop:}=}}

\newcommand{\mls}[1]{\mathit{#1}}

\newcommand{\inum}[1]{\num[round-precision=0,group-minimum-digits=3,group-separator={,}]{#1}}

\newcommand\snum[1]{\num[
scientific-notation=true,
round-precision=2,
fixed-exponent=1,
detect-weight=true,
detect-family=true,
round-mode=places,
retain-explicit-plus=true,
mode=text,fixed-exponent=0,
retain-explicit-plus=true,
output-exponent-marker=\text{e}]{#1}}

\newcommand{\tabadjust}{\centering\footnotesize}

\title[Fast $k$-means clustering in Riemannian manifolds via Fr\'{e}chet maps]{Fast $k$-means clustering in Riemannian manifolds via Fr\'{e}chet maps: Applications to large-dimensional SPD matrices}

\author{Ji Shi}
\author{Nicolas Charon}
\author{Andreas Mang}
\author{Demetrio Labate}
\author{Robert Azencott}
\makeatletter
\def\@setaddresses{} 
\makeatother

\thanks{This work was partly supported by the National Science Foundation (NSF) through the grant DMS-2438562 (NC) and DMS-2145845 (AM) and by the Simons Foundation grant MPS-TSM-00002738 (DL). Any opinions, findings, conclusions, or recommendations expressed herein are those of the authors and do not necessarily reflect the views of the NSF. This work was completed in part with resources provided by the Research Computing Data Core at the University of Houston.}

\date{\today}

\begin{document}
\maketitle
\begin{center}
\vspace{-0.8cm}
\small\textit{Department of Mathematics, University of Houston, Houston, TX, US}
\end{center}

\begin{abstract}
We introduce a novel, efficient framework for clustering data on high-dimensional, non-Euclidean manifolds that overcomes the computational challenges associated with standard intrinsic methods.  The key innovation is the use of the  $p$-\textit{Fr\'{e}chet map} $F^p : \mathcal{M} \to \mathbb{R}^\ell$ -- defined on a generic metric space $\M$ -- which embeds the manifold data into a lower-dimensional Euclidean space $\mathbb{R}^\ell$ using a set of reference points $\{r_i\}_{i=1}^\ell$, $r_i \in \M$. Once embedded, we can efficiently and accurately apply standard Euclidean clustering techniques such as k-means. We rigorously analyze the mathematical properties of $F^p$ in the Euclidean space and the challenging manifold of $n \times n$ symmetric positive definite matrices $\mls{SPD}(n)$. Extensive numerical experiments using synthetic and real $\mls{SPD}(n)$ data demonstrate significant performance gains: our method reduces runtime by up to two orders of magnitude compared to intrinsic manifold-based approaches, all while maintaining high clustering accuracy, including scenarios where existing alternative methods struggle or fail.
\end{abstract}

\section{Introduction}\label{s:intro}

Cluster analysis is one of the most common tasks in modern data science. Given a set of observations in some data space, being able to identify specific subgroups based on similarity patterns across the dataset is often a key stepping stone for data exploration and further analysis. Unlike classification, clustering is a fundamentally unsupervised problem as no labels are available a priori. Thus, the extraction of clusters can only be achieved by relying on some notion of proximity between data points, which is typically measured via a certain distance defined on the data space. Perhaps the most widely used general clustering method is the well-known \textit{k-means} algorithm, in which the partition of the dataset is built to minimize its resulting dispersion. Although this problem cannot be solved exactly in polynomial time, several greedy approximate schemes have been proposed, in particular Lloyd's celebrated k-means algorithm \cite{lloyd1982}, which is widely used in applications. One clear upside of Lloyd's algorithm is its simplicity: the whole scheme consists of a cluster assignment step followed by an update of the cluster centroids, which are iterated until stabilization of the clusters. The first step is typically performed by assigning each data point to the closest current cluster centroid as measured by the distance in the data space. The second step involves the computation of the centroid of each of the clusters at the current iteration. For data living in an Euclidean space, the latter simply reduces to computing the mean of each cluster, and thus both steps of Lloyd's scheme have closed-form updates, making the algorithm very fast to run even on spaces of large dimensions.

However, some obvious difficulties arise when, instead, data points live in a non-Euclidean manifold $\M$, for which the computation of the distance between two given points may require solving a geodesic boundary value problem. More importantly, even though theoretical extensions of the notion of centroid or mean (known as Fr\'{e}chet or Karcher mean) exist for a certain class of manifolds $\M$, these are highly non-trivial and almost always involve numerically costly optimization on $\M$. Due to the need for repeated evaluations of such distances and means in Lloyd's approach, its direct adaptation to the manifold setting, often referred to as intrinsic k-means~\cite{tan2024intrinsic}, can easily become prohibitively expensive, even more so on high-dimensional manifolds. Yet manifold data is increasingly common in many applications. In robotics, for instance, pose configurations are typically modeled as products of elements in the manifold of 3D rotations. In signal/image analysis, one is often interested in correlation or covariance matrices of signals, which naturally live in the manifold of symmetric positive definite matrices $\mls{SPD}(n)$. The fields of shape analysis and computer vision also consider data points (such as curves or surfaces) in inherently nonlinear spaces. We refer the reader to recent surveys such as~\cite{pennec2019riemannian} for more highlights on the growing importance of manifold data, and its associated challenges.

This paper introduces an efficient approach for k-means clustering of potentially large-dimensional non-Euclidean data, which we coin \emph{Fr\'{e}chet Map Clustering} ({\bf FMC}). The core idea is to map the data into some (preferably smaller) Euclidean space in which standard Euclidean clustering techniques can be applied at much lower cost compared to the original manifold $\M$. Specifically, our proposed approach relies on the special family of $p$-\textit{Fr\'{e}chet maps} $F^p : \mathcal{M} \to \mathbb{R}^\ell$ defined on a generic metric space $\M$ with distance $d_\M$ and parameterized by a set of reference points in $\M$. While this is an established strategy in modern machine learning with architectures such as autoencoders that are designed to learn a latent space data embedding tailored to a given task, we follow here a different paradigm by focusing on a particular and more restrictive, albeit more interpretable, class of mappings for which no prior training phase is needed.

Although in principle FMC can be applied to broad classes of manifolds or metric spaces, our main focus in this work is on the Riemannian manifold of $n \times n$ symmetric positive definite ({\bf SPD}) matrices $\mls{SPD}(n)$. SPD matrices, in particular, correlation matrices, appear in a range of applications, including high-dimensional statistics, image analysis, multi-sensor monitoring, and communication networks. In neuroimaging, for instance, correlation matrices of high dimensions derived from diffusion tensor imaging or functional magnetic resonance imaging are often employed to model the strength of neural connections between different brain sites and to assess brain function in normal and disease states~\cite {dryden2009non, hekmati2020, you2021}. In many of these problems, one is faced with the task of clustering large sets of SPD matrices, for example, high-dimensional correlation matrices, and we shall validate the potential of the FMC approach in this setting.

\subsection{Contributions}

We present a new method for efficiently clustering large-dimensional non-Euclidean data by embedding them into a Euclidean space using $p$-Fr\'echet maps, with a focus on the cases $p=1$ and $p=2$. The key contributions of this work are the following.
\begin{itemize}
\item We proved several fundamental properties of the $p$-Fr\'echet map $F^p$ (for general $p$ and, in particular, for the specific choices of $p=1,2$) when defined on a Euclidean space. These include the differentiability of $F^p$ and control bounds on the induced distortion, as well as sufficient conditions under which $F^p$ is a diffeomorphism onto its image. We further analyze the question of separability of the image of balls under Fr\'{e}chet maps, a key property when performing clustering in the image space.
\item
We examined the extension of this theoretical analysis to the significantly more challenging case of the $\mls{SPD}(n)$ manifold, highlighting in particular the important questions remaining open in this setting.
\item
We proposed a principled, yet practical strategy for selecting the reference points that control the map $F^p$ while maximizing stability and clustering accuracy.
\item
We empirically evaluated the proposed method against existing approaches for clustering data on SPD manifolds using both synthetic data and experimental data.
\end{itemize}

Our results confirm that {\bf the proposed methodology can efficiently cluster data in Riemannian manifolds without significantly sacrificing accuracy}. The proposed approach is up to {\bf two orders of magnitude faster} than the standard intrinsic $k$-means on the manifold $\M$, and we obtain {\bf runtimes that are competitive with similar strategies} operating in the tangent space, while {\bf producing high-accuracy clustering results more consistently}.

\subsection{Limitations}

Despite having established some clear advantages of the FMC framework, there are unresolved issues that will be the subject of future work, including the following.
\begin{itemize}
\item
Although we have a rather complete picture of the mathematical properties of Fr\'echet maps in the Euclidean case, there are important theoretical gaps to generalize some of those results to $\mls{SPD}(n)$ and to the broader class of \textit{Cartan--Hadamard manifolds}.
\item
We empirically observed that the choice of reference points has a significant impact on the performance of FMC. Although we have established a principled heuristic for selecting reference points, which yields very competitive results, additional work is required to derive strategies with rigorous theoretical guarantees.
\end{itemize}

\subsection{Related Work}

In recent years, there has been significant interest in the design of geometry-aware methods for classifying or clustering data in non-Euclidean, high-dimensional spaces. Due to the potential intricacies and numerical cost associated with some basic operations in manifolds (the estimation of a mean, for instance), a common workaround consists of embedding data points into some Euclidean space in a way that preserves as much of the original structure in the dataset as possible. Some very popular strategies involve the use of metric multidimensional scaling (MDS) \cite{mead1992review} or t-distributed stochastic neighbor embedding (t-SNE) \cite{hinton2002stochastic}. A downside of both of these methods, however, is that they require the computation of all distances between every pair of data points, which can represent a non-negligible amount of computations in the manifold case. Furthermore, recent works such as \cite{bergam2025t} have suggested that t-SNE embeddings may, in some situations, generate artefactual clusters.

Another approach, which is more closely related to ours, especially with regard to $\mls{SPD}(n)$, is the \textit{log-Euclidean} framework of~\cite{arsigny2007geometric, pennec2020manifold}. In its standard form, the log-Euclidean setting consists of a linearized tangent space approximation of the metric at a specific template point (usually the Fr\'{e}chet mean of the dataset). This results in a fast-to-compute and often efficient data embedding strategy to perform clustering. We demonstrate that FMC is competitive with the log-Euclidean framework in terms of runtime and clustering accuracy. More importantly, we show empirically that it maintains good clustering accuracy in settings where the performance of the log-Euclidean framework deteriorates.

Lastly, we will also compare our approach to the generic \emph{Intrinsic Riemannian Clustering}, i.e., the direct transposition of Lloyd's algorithm to the SPD manifold setting. For the computation of the clusters' Fr\'{e}chet means, in addition to the standard gradient descent method from~\cite{afsari2013convergence}, we also consider alternative faster approximation schemes such as the \emph{iterative centroid method proposed} in~\cite{ho2013recursive,cury2013template} based on the recursive scheme from~\cite{sturm2003probability}.

\subsection{Layout}

We provide background material in \Cref{s:background}. We start by briefly discussing the problem of $k$-means clustering for manifold data in \Cref{s:k_means_manifold}. We introduce the general notion of Fr\'echet mapping in \Cref{s:frechetmap}. We explore the properties of the Fr\'echet mapping for $\mls{SPD}(n)$ in \Cref{s:spd-mat}. This includes the structure of $\mls{SPD}(n)$ (see \Cref{s:spd_affmetric}), properties of $F^p$ for $\mls{SPD}(n)$ (see \Cref{s:frechet_spd_properties}), a comparison of different clustering methods in $\mls{SPD}(n)$ (see \Cref{s:clustering_spd}). We describe our approach for selecting reference points for $F^p$ in \Cref{s:ref_selection}. Numerical results are reported in \Cref{s:results}. We conclude with \Cref{s:conclusions}.

\section{Background and general framework}\label{s:background}

Before we provide some theoretical insights, we introduce the problem setting as well as notations, and define the Fr\'echet map more precisely.

\subsection{k-means clustering for manifold data}
\label{s:k_means_manifold}

Clustering is a fundamental problem of machine learning. From a set of observations in a data manifold $\M$ its goal is to partition these observations into a set of meaningful clusters $\P = (\mls{CL}_1,\dots,\mls{CL}_k)$, where each $\mls{CL}_j, j = 1, \ldots, k$, is determined by measuring the proximity between observations using a distance on $\M$.

Among existing clustering approaches, the $k$-means algorithm~\cite{macqueen1967some} remains among the most widely used methods due to the simplicity of its formulation and its ability to easily adapt to different types of data. Given a finite set of observations $\D = \{x_i\}_{i=1}^n \subset \M$ and a fixed number $k \in \mathbb{N}$ of target clusters, the $k$-means algorithm looks for a partition $\P$ that minimizes the \textit{total dispersion} defined by
\begin{equation}
\label{eq:totaldisp}
\operatorname{totdisp}(\P)
= \sum_{i=1}^k \operatorname{disp}(\mls{CL}_i)
= \sum_{i=1}^k \frac{1}{|\mls{CL}_i|}\sum_{x\in \mls{CL}_i} d_{\M}(x,c_i)^2,
\end{equation}
where $d_{\M} : \M \times \M \to \mathbb{R}$ denotes the distance on the data manifold being considered, $c_i \in \M$ is the centroid or ``mean'' of the cluster $\mls{CL}_i$ in the partition $\P$, and $|\mls{CL}_i|$ is the cardinality of the cluster $\mls{CL}_i$. To make the above definition precise, one needs to make more specific assumptions on the structure of the data manifold $\M$. A typical setup is to consider a Riemannian manifold $\M$ with a corresponding Riemannian distance $d_\M$, in which case one can extend the notion of Euclidean mean via the so-called Fr\'{e}chet mean (also known as the Fr\'{e}chet mean). Leaving aside for now the question of the existence and uniqueness of the Fr\'{e}chet mean in a Riemannian manifold, the total dispersion in \Cref{eq:totaldisp} can be interpreted as the sum of the variances of each cluster.

Although there is only a finite number of possible partitions $\P$ of $\D$, this number grows exponentially with the number of observations, so that finding a global minimum of $\operatorname{totdisp}(\P)$ is an NP-hard problem. A practical alternative to address this problem is the greedy iterative approach known as Lloyd's algorithm~\cite{lloyd1982}, which is generically referred to as the $k$-means algorithm. The classical Lloyd's algorithm in $\mathbb{R}^m$ is given as \Cref{algo:lloyd}.

\begin{algorithm}
\caption{Lloyd's algorithm~\cite{lloyd1982}.} \label{algo:lloyd}
\begin{algorithmic}[1]
\STATE {\bf Input: } A set of points $\D = \{x_1,\dots, x_N\} \subset \mathbb{R}^m$ and an initialization of the cluster centroids $C = \{c_1,\ldots,c_k\} \subset \mathbb{R}^m$
\STATE stop $\gets$ false
\WHILE{$\neg$ stop}
\STATE Assign each $x \in \D$ to the cluster with closest center, i.e., $x \in \mls{CL}_i$ for
\begin{equation}\label{i:assignment}
i= \operatorname{arg\,min}_{j=1,\ldots,k} \ d(x,c_j),
\end{equation}
\noindent where $d$ is the Euclidean distance in $\mathbb{R}^m$
\STATE Recalculate the cluster centroids $C$ by setting for each $i=1,\dots,k$,
\begin{equation} \label{i:centroid-update}
c_i \gets \frac{1}{|\mls{CL}_i|} \sum_{x \in \mls{CL}_i} x
\end{equation}
\STATE stop $\gets$ check convergence
\ENDWHILE
\STATE {\bf Output:} Cluster centroids $C = \{c_1,\ldots,c_k\}$ and corresponding clusters $\mls{CL}_1,\dots, \mls{CL}_k$
\end{algorithmic}
\end{algorithm}

\begin{remark}
\label{rem:k_means_hyperplane}
Since the cluster assignment in Algorithm~\ref{algo:lloyd} is based on the proximity to the cluster centroids $c_i$'s, it leads to a partition of the whole space $\R^m$ delimited by $k(k+1)/2$ \emph{mediatrix} hyperplanes $\{x: \|x-c_i\| = \|x-c_j\|\}$ for all pairs of cluster centers $(c_i,c_j)$. This shows that the algorithm separates the different clusters through affine hyperplanes in $\R^m$ that are the \emph{mediatrices} of the cluster centroids.
\end{remark}

Lloyd's algorithm in $\mathbb{R}^m$ can be shown to converge, but not necessarily to a global minimizer of the total dispersion. In practice, it is common to run the algorithm for multiple different initializations of the centroids and ultimately select the solution with the lowest total dispersion, which generally provides a good estimate of the true solution.

However, extending Lloyd's algorithm from $\mathbb{R}^m$ to a more general Riemannian manifold $\M$ is not as straightforward. Specifically, the centroid computation, which is simply an arithmetic mean in \Cref{i:centroid-update}, must be replaced by a Fr\'{e}chet mean on $\M$ and typically requires additional conditions on the manifold or the dataset $\D$.

Indeed, the existence and uniqueness of the Fr\'{e}chet mean of a set of points in $\M$ is not guaranteed for general Riemannian manifolds. This can be ensured either by assuming that data points in $\D$ are sufficiently concentrated or by considering a more specific structure for $\M$. A particularly well-suited class of manifolds in that regard is the \textit{Cartan--Hadamard manifolds}, which are the simply connected complete Riemannian manifolds with non-positive sectional curvature. In this case, one can ensure the existence and uniqueness of geodesics between any two points in $\M$ as well as the existence and uniqueness of Fr\'{e}chet means for any set of points in $\M$, cf. \cite{shiga1984,petersen2006riemannian,afsari2011riemannian}. Cartan--Hadamard manifolds encompass many interesting examples of data manifolds beyond Euclidean spaces that are found in applications, including the space of SPD matrices, which will be the focus of this paper.

Despite the Cartan--Hadamard manifold providing an adequate theoretical setting to extend the $k$-means algorithm to manifolds, there are still important practical difficulties compared to the Euclidean case. For some of those manifolds, an explicit expression of the geodesic distance may not be available, in which case one needs to solve a geodesic search problem to compute a distance between two points. More importantly, even in cases where the distance can be evaluated in a closed form (such as $\mls{SPD}(n)$, c.f. \Cref{s:spd_affmetric} below), there is in general no closed-form solution for the Fr\'{e}chet mean of a given set of points in $\M$. As a result, to find a centroid in~\Cref{i:centroid-update}, one needs to find an approximate solution of
\begin{equation}
\label{eq:Frechet_mean_opt}
    c_i = \underset{p \in \M}{\text{argmin}} \sum_{x \in \mls{CL}_i} d_{\M}(x,p)^2.
\end{equation}

In a Cartan--Hadamard manifold, \eqref{eq:Frechet_mean_opt} is a convex problem which can be tackled using various optimization strategies, the simplest one being the standard gradient descent method, c.f.~\cite{afsari2013convergence}. However, due to the large number of Fr\'{e}chet means that need to be estimated throughout the iterations of $k$-means, this computation could become prohibitively costly for clustering. Some alternative strategies to solve~\cref{eq:Frechet_mean_opt} include stochastic gradient descent~\cite{arnaudon2012stochastic,bonnabel2013stochastic} or the recursive barycenter scheme~\cite{sturm2003probability,ho2013recursive}. Even these approaches are usually not sufficient on their own to make $k$-means clustering efficiently tractable for large-dimensional manifolds. One of the key contributions of our work is a novel methodology that addresses this challenge.

\subsection{The Fr\'echet mapping}\label{s:frechetmap}

In this section, we present the main idea underpinning the proposed approach.

\subsubsection{General setting and basic properties}\label{s:properties-frechet}

The \textit{p-Fr\'{e}chet map} $F^p$ on a generic metric space $\M$ with distance $d_\M$ is defined as follows.

\begin{definition}
Let $(\M,d_\M)$ be a metric space and fix a set of $\ell$ points $\{r_1,\ldots,r_\ell \} \subset \M$. For $p \in \mathbb{R},\ p \geq 1$, we define the $p$-Fr\'{e}chet map associated with the list of reference points $r = (r_1,\dots,r_\ell)$
by
\begin{eqnarray*}
F_r^p: \M &\to& \mathbb{R}_+^\ell\\
x  &\mapsto& (d_\M(r_1,x)^p,\ldots,d_\M(r_\ell,x)^p).
\end{eqnarray*}
\end{definition}
In some cases, when the dependency on the reference point or order $p$ does not need to be emphasized, we shall abbreviate the notation for the Fr\'{e}chet map to $F^p$ or even simply $F$.

In this work, we will be interested in two particular values of $p$, namely $p=1$ and $p=2$. The case $p=1$ corresponds to the standard definition of the Fr\'{e}chet map introduced in the field of discrete geometry~\cite{bourgain1985lipschitz, matousek2013lectures}, where it has played a particular role in the construction of quasi-isometric embeddings of finite metric spaces into Euclidean spaces. The case $p=2$ corresponds to the squared Fr\'{e}chet mapping and offers the advantage of being differentiable everywhere when $\M$ is a smooth Riemannian manifold, together with certain convexity properties. We note that the general $p$-Fr\'{e}chet map is well-defined for any metric space $(\M,d_\M)$ as the only requirement is that the distance between any two points exists.

Following the idea of performing data clustering by applying the k-means algorithm in $\mathbb{R}_+^\ell$ on the images of the data points via $F_r$, it is important to investigate the properties of the $p$-Fr\'{e}chet maps. Our first observation is that, in a general metric space $\M$,  for any choice of reference points, the map $F_r^p$ is locally Lipschitz.

\begin{proposition}\label{prop:Lipschitz_reg_Frechet}
Let $F_r^p$, $p \ge 1$, be a $p$-Fr\'{e}chet map on a metric space $(\M,d_\M)$ associated with a list of reference points $r = (r_1,\dots,r_\ell)$ in $\M$.
For $p=1$, the Fr\'{e}chet map $F_r^1$ is a globally Lipschitz map on $\M$. If for some $x_0 \in \M$ and $\delta >0$, one has $r_1,\ldots,r_\ell \in B(x_0,\delta)$ then for any $x,x' \in B(x_0,\delta)$, it holds that
$\|F_r^p(x) - F_r^p(x')\|_{\infty} \leq p \, 2^{p-1} \, \delta^{p-1} \, d_\M(x,x')$ for any $p \ge 1$.
\end{proposition}
\begin{proof}
  From the triangle inequality, we immediately see that $|d(x,r_i) - d(x',r_i)|\leq d(x,x')$ for all $i=1,\ldots,\ell$, and thus $\|F_r^1(x) - F_r^1(x')\|_{\infty} \leq d_\M(x,x')$ which proves that $F_r^1$ is $1$-Lipschitz on $\M$. For $p \ge 1$, we observe that for each $i$:
  \begin{equation*}
      \begin{aligned}
          |d_\M(x,r_i)^p - d_\M(x',r_i)^p| &\leq p \max\{d_\M(x,r_i),d_\M(x',r_i)\}^{p-1} |d_\M(x,r_i) - d_\M(x',r_i)| \\
          &\leq p \max\{d_\M(x,r_i),d_\M(x',r_i)\}^{p-1} d_\M(x,x').
      \end{aligned}
  \end{equation*}
  Since $d_\M(x,r_i) \leq d_\M(x,x_0) + d_\M(x_0,r_i) \leq 2 \delta$ and similarly for $x'$, we obtain the stated upper bound.
\end{proof}

\Cref{prop:Lipschitz_reg_Frechet} guarantees that the distortion induced by $F_r^p$ from the original distance in $\M$ to the Euclidean norm in the image space $\R^\ell_+$ remains controlled for data within a geodesic ball and even globally controlled for $p=1$.

When $\M$ is a smooth Cartan--Hadamard manifold, it is well-known that each squared distance function $x\mapsto d_\M(p,x)^2$, for $p \in \M$, is differentiable on $\M$ and its Riemannian gradient is given by $-2\log_{x} p$, where $\log_x : \M \rightarrow T_x \M$ is the logarithm map of $\M$ at the foot point $x$ (which induces a diffeomorphism between $\M$ and $T_x\M$); c.f.~\cite{pennec2018barycentric}. This implies the following result.

\begin{lemma}\label{lemma:diff_Frechet}
Let $\M$ be a Cartan--Hadamard manifold with Riemannian distance $d_\M$. Any Fr\'{e}chet map $F_r^p$ on $(\M,d_\M)$ with a list of reference points $r=(r_1,\ldots,r_\ell)$ is differentiable on $\M$ for $p\geq 2$ and on $\M\backslash \{r_1,\ldots,r_\ell\}$ for $1\leq p <2$ with the Riemannian Jacobian matrix given, on those respective sets, by:
\begin{equation*}
DF_r^p(x) = (-p d(r_i,x)^{p-2}\log_{x} r_1, \ -p d(r_i,x)^{p-2}\log_{x} r_2, \ldots, \ -p d(r_i,x)^{p-2}\log_{x} r_\ell).
\end{equation*}
\end{lemma}

\noindent In particular, $F_r^1$ is not differentiable at the reference points as opposed to $F_r^2$. We illustrate in the following section useful important properties of the $p$-Fr\'{e}chet maps with $p=1,2$ in the context of the clustering problem considered in this paper when $\M$ is a Euclidean space.

\subsubsection{$p$--Fr\'{e}chet maps on Euclidean space}\label{s:frechet_euclidean}

In this section, we examine the situation in which $\M$ is the space $\R^m$ equipped with the usual Euclidean metric. Although this case is not of practical interest when it comes to the clustering framework we propose, this analysis provides useful insights in preparation to~\Cref{s:frechet_spd_properties}.

To simplify the exposition, we focus on the squared Fr\'{e}chet map $F_r^2$ and mention, in remarks, the case of $F_r^1$ when relevant. If $r=(r_1,\ldots,r_\ell) \in \R^{m\times \ell}$ are given reference points, the associated Fr\'{e}chet map is $F_r^2(x) = (\|x-r_1\|^2,\ldots,\|x-r_\ell\|^2)^{\mathsf{T}}$ with differential $DF_r^2(x) = 2(x-r_1,\ldots,x-r_\ell)$. From this, we immediately see that the rank of $DF_r^2(x)$ at $x \in \R^m$ is the dimension of $\operatorname{Span}(x-r_1,\ldots,x-r_\ell)$. If $DF^2_r(x)$ is full-rank, by the local inverse function theorem, $F_r^2$ is locally invertible from a neighborhood of $x$ to its image. In this case, $DF_r^2(x)$ is rank deficient if and only if the reference points lie in some affine subspace of dimension strictly smaller than $m$ passing through $x$. It follows, conversely, that $DF_r^2(x)$ is full rank when either one of the following two conditions is satisfied:
\begin{enumerate}
    \item $\ell\geq m+1$ and the affine hull of the reference points is such that $\operatorname{Aff}(r_1,\ldots,r_\ell) = \R^m$.
    \item $\ell=m$, the reference points are affinely independent (i.e., $\operatorname{Aff}(r_1,\ldots,r_\ell)$ is of dimension $m-1$) and $x \notin \operatorname{Aff}(r_1,\ldots,r_\ell)$.
\end{enumerate}

The first situation corresponds to having more reference points than the dimension of the space placed in a ``generic position.'' In this case, $F_r$ is locally invertible at each point $x \in \R^m$. The second situation consists of taking a generic configuration of exactly $m$ reference points. In this case, the Fr\'{e}chet map is locally invertible at each point outside of the $(m-1)$-dimensional affine hull of the reference points.

The next natural question is whether $F_r^2$ is also a globally injective map. We reason by contradiction; if we assume that $F_r^2$ is not injective on $\R^m$, then there exist $x\neq x'$ such that $F_r^2(x) = F_r^2(x')$, i.e., $\|x-r_i\| = \|x'-r_i\|$ for any $i=1,\ldots,\ell$. Geometrically, this means that $r_1,\ldots,r_\ell$ all lie on the mediatrix hyperplane $H = \{z\in \R^m: \|x-z\|=\|x'-z\|\}$. This immediately implies that with $\ell \geq m+1$ and the reference points in generic position, i.e., $\operatorname{Aff}(r_1,\ldots,r_\ell) = \R^m$ as above, the Fr\'{e}chet map $F_r^2$ is necessarily injective. Let $H^0 = \operatorname{Aff}(r_1,\ldots,r_\ell)$ denote the affine hull of affinely independent reference points $r_1,\ldots,r_\ell$.  When $\ell = m$, we know that $H^0$ is an affine subspace of dimension $m-1$ that divides $\R^m$ into two disjoint halfspaces $H^{-}$ and $H^+$. In this case, it is easy to see that any two points $x,x'$, which are symmetric with respect to the hyperplane $H^0$, satisfy $F_r^2(x) = F_r^2(x')$ and thus $F_r^2$ cannot be globally injective on $\R^m$. However, $F_r^2$ is injective on both halfspaces $H^{-}$ and $H^{+}$: if not, using the same mediatrix argument as previously, the reference points would all lie at the intersection of $H_0$ and some other transverse hyperplane, which is impossible by assumption.

\begin{figure}
\centering
\begin{tabular}{ccc}
\includegraphics[height=5cm]{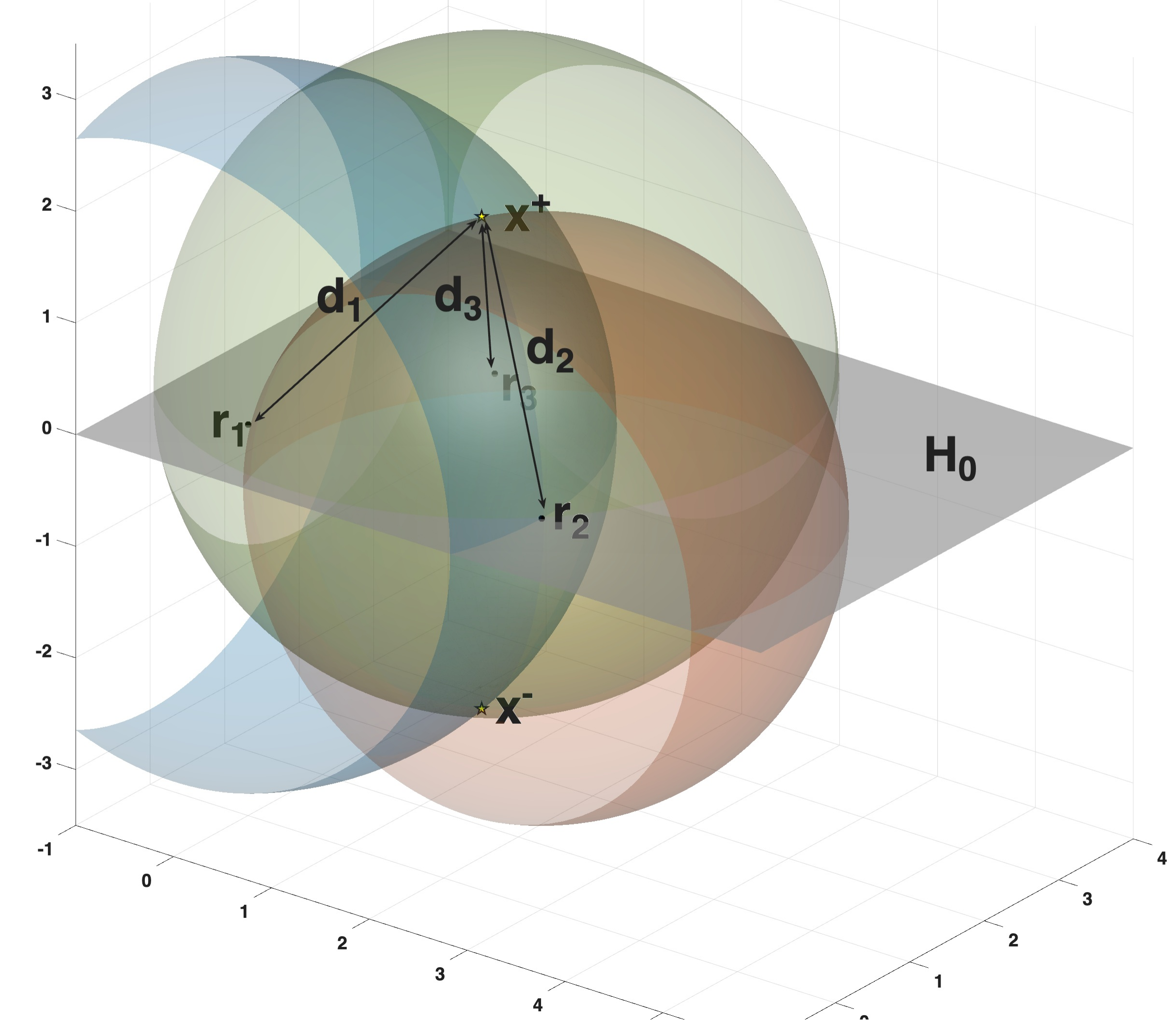}
& &
\includegraphics[height=5cm]{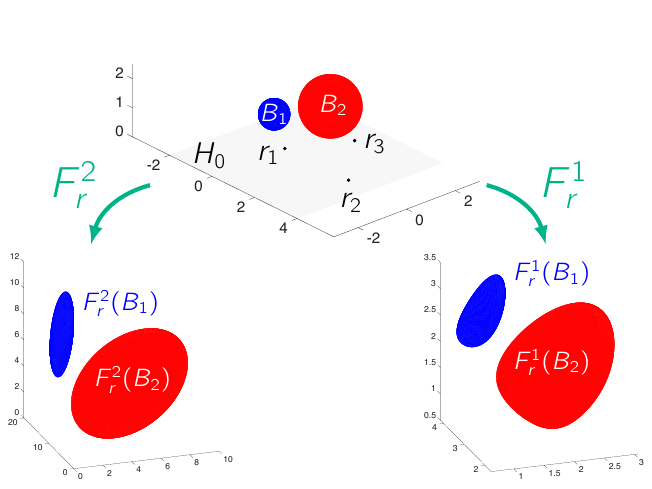} \\
(a) & & (b)
\end{tabular}
\caption{Illustration of the Fr\'{e}chet map in the Euclidean space $\M=\R^3$ with $\ell=3$ reference points. The left panel (a) shows the two symmetric points with the same given image $d=(d_1,d_2,d_3)$ by the Fr\'{e}chet map. The right panel shows the images of two disjoint balls in the upper halfspace by the Fr\'{e}chet maps $F_r^2$ and $F_r^1$.}\label{Frechet_Euc.fig}
\end{figure}

Summarizing the previous paragraphs, we have obtained the following general property for Fr\'{e}chet maps on Euclidean spaces.
\begin{theorem}
\label{thm:invertibility_Frechet_Eucl}
Let $F_r^2$ be a Fr\'{e}chet maps on $\R^m$ with reference points $r=(r_1,\ldots,r_\ell) \in \R^{m\times \ell}$. The following holds:
\begin{enumerate}
\item if $\ell\geq m+1$ and $\operatorname{Aff}(r_1,\ldots,r_\ell) = \R^m$ then $F_r^2$ is a diffeomorphism from $\R^m$ to its image $F_r^2(\R^m)\subset \R^{\ell}$;
\item if $\ell=m$ and the reference points are affinely independent, then $F_r^2$ is a diffeomorphism on each of the two halfspaces delimited by the affine hyperplane $H^0 = \operatorname{Aff}(r_1,\ldots,r_m)$.
\end{enumerate}
\end{theorem}

Under either of the two configurations of reference points described above, one could ask what is specifically the image of the Fr\'{e}chet map and how to derive its inverse map. This amounts to reconstructing the position of a point given its distances to the fixed set of reference points. This problem is known as the \textit{multilateration} problem in the literature; it is connected to applications in GPS positioning. Geometrically, it can be seen as finding the intersection of $\ell$ spheres centered at the reference points. \Cref{Frechet_Euc.fig}(a) depicts an illustration in $\R^3$ with $\ell = 3$ reference points, and the two resulting solutions that are symmetric with respect to the hyperplane $H^0$. Mathematically, the problem can be cast as a system of quadratic equations on the coordinates of the point $x \in \R^m$. This problem can be approached in various ways~\cite{bancroft2007algebraic}. Since the proposed Fr\'{e}chet map clustering approach does not explicitly require the computation of the inverse map, we do not elaborate further on this point. However, we note that some partial derivation for $\ell =m$ can be found as part of the proof of \Cref{app:proof_image_balls}.

\begin{remark}
\label{rem:invertibility_Frechet_Eucl}
\Cref{thm:invertibility_Frechet_Eucl} and the above statements remain nearly identical for the 1-Fr\'{e}chet map. Specifically, when $\ell \geq m+1$, $F_r^1$ is a homeomorphism from $\mathbb{R}^m$ to $F_r^1(\mathbb{R}^m)$ and a diffeomorphism on the subset $\mathbb{R}^m \backslash\{r_1,\ldots,r_\ell\}$. For $\ell = m$, the result is the same as point 2 of the theorem.
\end{remark}

A direct consequence of \Cref{thm:invertibility_Frechet_Eucl} is that, under the stated assumptions, two disjoint subsets $S_1,S_2 \subset \R^m$ have disjoint images in $\R^\ell$ under the Fr\'{e}chet map. Yet, as it comes to using Fr\'{e}chet maps in combination with $k$-means, the key question is to determine under which conditions $F_r(S_1)$ and $F_r(S_2)$ are in addition separable by a hyperplane of $\R^\ell$. This is a necessary condition for $k$-means applied to $F_r^2(S_1)\cup F_r^2(S_2)$ to have a solution that perfectly clusters those two sets (see \Cref{rem:k_means_hyperplane}). This can be guaranteed in the case of $\ell=m$ reference points thanks to the following result.

\begin{theorem}\label{thm:convex_image_Frechet}
Let $F_r^2$ be a Fr\'{e}chet map on $\R^m$ with reference points $r=(r_1,\ldots,r_m) \in \R^{m\times m}$ and assume that the reference points $\{r_1,\ldots,r_m \}$ are affinely independent in $\R^m$. Then the image $F_r^2(\R^m)$ is the interior of a paraboloid of $\R^m$. Moreover, if $B$ is a ball contained in either $H^-$ or $H^+$, its image $F_r^2(B)$ is an ellipsoid. Consequently, given two disjoint balls $B_1$ and $B_2$ both included in one of the two half-spaces $H^-$ or $H^+$, their images $F_r^2(B_1)$ and $F_r^2(B_2)$ are separable by a hyperplane in $\R^m$.
\end{theorem}

The proof is included in \Cref{app:proof_image_balls}. This result shows that, despite the distortion induced by the Fr\'{e}chet map, $F_r^2$ still preserves the convexity of balls in $\M$ and thus their separability by hyperplanes in the image space. \Cref{Frechet_Euc.fig}(b) shows a particular example for the case $m=3$. We point out that this is specific to the case of the 2-Fr\'{e}chet map: for the 1-Fr\'{e}chet map, the image of a ball no longer has a simple ellipsoidal geometry, as also shown in \Cref{Frechet_Euc.fig}(b). We can still show the following weaker result, which is proved in \Cref{app:proof_convex_image_Frechet1}:
\begin{theorem}\label{thm:convex_image_Frechet1}
Let $A$ be a compact subset of $\R^m$. Assume that $r=(r_1,\ldots,r_m)$ are $m$ reference points in $\R^m$ where $\{r_1,\ldots,r_m \}$ are affinely independent in $\R^m$, and such that either $A\subset H^{-}$ or $A \subset H^{+}$. If $B$ is any ball of radius $\rho>0$ with $B \subset A$ and the reference points satisfy the condition:
\begin{equation*}
    \frac{\rho}{d(r,A)} < \frac{1-(m-1)\mu}{\sqrt{m}}
\end{equation*}
where $d(r,A)$ is the distance from the reference point set to $A$ and $\mu$ denotes the largest mutual coherence over $A$ i.e. $\mu = \max_{x \in A} \max_{i\neq j} \left| \langle\frac{x-r_i}{\|x-r_i\|}, \frac{x-r_j}{\|x-r_j\|}\rangle \right|$, then $F_r^1(B)$ is convex in $\R^m$.
\end{theorem}
We see that the above sufficient condition for convexity is satisfied for reference points chosen sufficiently far from $A$ while also requiring those reference points to be sufficiently spread apart in order to control the mutual coherence $\mu$. However, it is likely that less restrictive conditions could be obtained, but we leave it an open question for future work.

The main problem we are interested in addressing is whether a similar picture can be obtained for Fr\'{e}chet maps on manifolds $\M$. Some analogous analysis can be done with manifolds of constant negative curvature, namely the hyperbolic spaces $\M = \mathbb{H}^{m}$. For the sake of brevity, we will not detail this case as it would require the introduction of many additional definitions and notations. In the next section, we shall instead focus on the manifold of SPD matrices of size $n \times n$, which we denote by $\mls{SPD}(n)$, that are connected to our applications of interest. As we show below, this class of manifolds already involves many key challenges in generalizing the results of \Cref{thm:invertibility_Frechet_Eucl} and \Cref{thm:convex_image_Frechet}, respectively.

\section{SPD matrices}\label{s:spd-mat}

In this section, we focus on the manifold $\mls{SPD}(n)$. This is a very insightful case to consider, as they are Cartan--Hadamard spaces in which explicit expressions for the geodesics and the exponential and logarithm maps are available. However, the curvature is not constant, which leads to considerable challenges in the context of the clustering framework considered in this work. Furthermore, several applications involve objects living in $\mls{SPD}(n)$, sometimes with a large dimension $n$, so that the Fr\'{e}chet map framework we propose in this paper is especially relevant in this setting.

\subsection{The manifold $\mls{SPD}(n)$ and its Riemannian structure}
\label{s:spd_affmetric}

We start by introducing some basic definitions and notation about $\mls{SPD}(n)$. First, as an open subset of the space of symmetric matrices $\mls{Sym}(n)$, $\mls{SPD}(n)$ can be viewed as a submanifold (of dimension $n(n+1)/2$) of $\mls{Sym}(n)$. Its tangent space at any $P \in \mls{SPD}(n)$ can be identified with $\mls{Sym}(n)$ itself.

Being in addition a convex subset of $\mls{Sym}(n)$, it may seem logical to equip $\mls{SPD}(n)$ with the restriction of the standard Euclidean (i.e., the Frobenius) distance between symmetric matrices. However, it is known that this simple metric can result in unwanted effects when computing averages and, by extension, when performing clustering of $\mls{SPD}$ matrices. It also lacks some fundamental invariance with the group action of the affine group $\mls{GL}(n)$ by conjugation. This invariance is a rather natural property when considering $\mls{SPD}$ matrices representing the covariance of processes, since it encodes the independence of the distance to the choice of coordinate system. Thus, defining adequate metrics on the manifold $\mls{SPD}(n)$ has been an important topic of research and has led to many different mathematical constructions. As our focus here is on Riemannian metrics and Cartan--Hadamard manifolds, we first discuss the widely used affine invariant metric proposed originally in~\cite{skovgaard1984riemannian}. It can be obtained as follows: First, one starts with the Euclidean metric on the tangent space at the identity matrix $\Id$. For any $V,W \in T_{\Id} \mls{SPD}(n) \approx \mls{Sym}(n)$, we define
\[
\langle V, W \rangle_{\Id} = \Tr(V^TW) = \Tr(VW).
\]

\noindent The idea is then to extend it to the whole $\mls{SPD}(n)$ via affine invariance. This means enforcing that the action of the affine group on $\mls{SPD}(n)$, given by $P\mapsto A^\mathsf{T} P A$ for any $A \in \mls{GL}(n)$, is by isometry. Then, if $P \in \mls{SPD}(n)$, one can write $P=P^{1/2} P^{1/2} = (P^{1/2})^\mathsf{T} \Id\, P^{1/2}$, where $P^{1/2} \in \mls{SPD}(n)$ is the SPD square root of $P$. In addition, the invariance property of the metric implies necessarily that
\begin{equation}
\label{eq:aff_inv_metric}
\langle V,W \rangle_P = \langle P^{-1/2} V P^{-1/2}, P^{-1/2} W P^{-1/2} \rangle_{\Id} = \Tr(P^{-1} V P^{-1} W).
\end{equation}

The properties of this metric have been extensively studied; see, for instance, \cite{pennec2020manifold} for a synthesis of those results. We recapitulate the main expressions used in the rest of the paper. First, one can show that the resulting Riemannian distance can be computed explicitly between any two matrices $P,Q \in \mls{SPD}(n)$ and is given by
\begin{equation}
\label{eq:SPD_distance}
 d(P,Q)^2 = \Tr(\log(P^{-1/2}QP^{-1/2})^2).
\end{equation}

\noindent Likewise, the constant-speed geodesic $\gamma:[0,1] \to \mls{SPD}(n)$ from $P$ to $Q$ is given by
\begin{equation}
\label{eq:SPD_geod}
\gamma(t) = P^{1/2} \exp(t \log(P^{-1/2} Q P^{-1/2})) P^{1/2} = P^{1/2} (P^{-1/2} Q P^{-1/2})^{t} P^{1/2}.
\end{equation}

As above, $\log$ and $\exp$ denote the matrix logarithm and exponential for symmetric positive definite matrices. Furthermore, the Riemannian exponential and logarithm maps on $\mls{SPD}(n)$ have the following expressions, which hold for any $P,Q \in \mls{SPD}(n)$ and any $V \in \mls{Sym}(n)$:
\begin{equation}\label{eq:aff_inv_exp_log}
\begin{aligned}
    \exp_P(V) &= P^{1/2} \exp(P^{-1/2} V P^{-1/2}) P^{1/2}\\
    \log_P(Q) &= P^{1/2} \log(P^{-1/2} Q P^{-1/2}) P^{1/2}.
\end{aligned}
\end{equation}

It follows that the manifold $\mls{SPD}(n)$ is complete for the affine-invariant metric. This is because the exponential map is always well-defined; it is a diffeomorphism from $\mls{Sym}(n)$ to $\mls{SPD}(n)$. Moreover, it can be shown (c.f., \cite{pennec2020manifold}, Theorem 3.3) that all sectional curvatures of $\mls{SPD}(n)$ are non-positive. This property makes $\mls{SPD}(n)$ a Cartan--Hadamard manifold.

\begin{figure}
\centering
\includegraphics[width=0.3\textwidth]{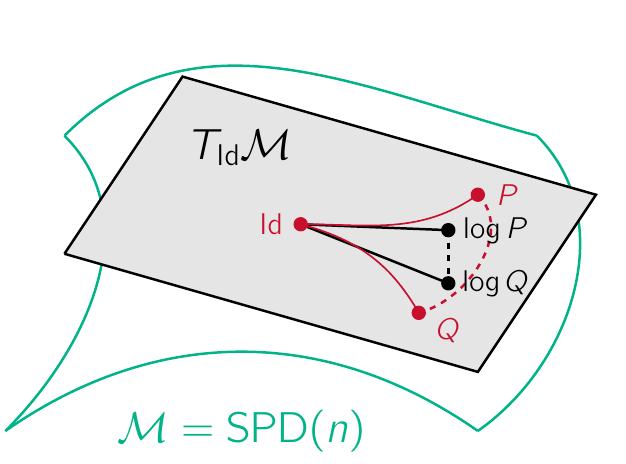}
\caption{Illustration of the Riemannian vs log-Euclidean metrics. We show the tangent space $T_{\Id}\mathcal{M}$ and two points $P$ and $Q$ on the manifold $\mathcal{M}$. Riemannian distances are shown in red. Euclidean distances in $T_{\Id}\mathcal{M}$ are shown in black. We project points $P$ and $Q$ to $T_{\Id}\mathcal{M}$ using a logarithmic map.}\label{log_Euc.fig}
\end{figure}

Although the above affine-invariant metric provides a relatively simple Riemannian structure on $\mls{SPD}(n)$, there are many alternative (typically non-Riemannian) metrics (or even divergences) that have been considered in the literature. We do not provide a comprehensive list of such metrics but specifically mention the case of the \textit{log-Euclidean} framework introduced in~\cite{arsigny2007geometric}, since it is a natural approach to compare to the Fr\'{e}chet mapping idea introduced in this work. In its standard form, the log-Euclidean metric can be seen as an approximation of the affine-invariant Riemannian distance centered at the identity matrix $\Id$ (or, more generally, at a chosen template point in $\mls{SPD}(n)$). Given two matrices $P,Q \in \mls{SPD}(n)$, we can represent them as elements of the tangent space at $\Id$ using the log map at $\Id$, which is simply the usual matrix logarithm. One can then compare the resulting symmetric matrices $\log P \in \mls{Sym}(n)$ and $\log Q \in \mls{Sym}(n)$ based on the metric on $T_\Id \mls{SPD}(n)$ (i.e., the Euclidean metric on $\mls{Sym}(n)$). This leads to the log-Euclidean distance
\begin{equation*}
    d_{\text{LE}}(P,Q) = \|\log P - \log Q\|_F.
\end{equation*}

In other words, the log-Euclidean distance on $\mls{SPD}(n)$ is obtained as the Frobenius norm between the logarithms of $P$ and $Q$. The general idea is illustrated in \Cref{log_Euc.fig}. Although this metric loses the full affine invariance of \Cref{eq:SPD_distance}, it retains the invariance with respect to the action of similarities of $\R^n$. The most relevant aspect of the log-Euclidean framework is that it allows for significantly faster computation compared to the geodesic distance. For example, performing $k$-means clustering can be reduced to applying standard Euclidean $k$-means to the logarithm of the data matrices, resulting in computational complexity comparable to the standard $k$-means algorithm in Euclidean spaces.

\subsection{Properties of the Fr\'echet map on $\mls{SPD}(n)$}
\label{s:frechet_spd_properties}

 Similarly to the discussion on the Euclidean case in \Cref{s:frechet_euclidean}, here we focus our analysis on the squared Fr\'{e}chet map ($p=2$) and only state the corresponding results for the $1$-Fr\'{e}chet map in remarks, as the arguments are similar. For brevity, throughout this section we drop from the notation the reference point set and the power $p$ and simply write the squared Fr\'echet map as:
\begin{equation*}
    F : P \in \mls{SPD}(n) \mapsto (d(R_1,P)^2,d(R_2,P)^2,\ldots,d(R_\ell,P)^2) \in \R_{+}^\ell,
\end{equation*}

\noindent with $d$ being the Riemannian distance defined in \Cref{eq:SPD_distance}; we denote the reference points, which are matrices in $\mls{SPD}(n)$, as $R_1, \dots, R_\ell$. It follows from \Cref{lemma:diff_Frechet} that $F$ is differentiable and the Riemannian differential is given, for any $P \in \mls{SPD}(n)$ and any $V \in \mls{Sym}(n)$, by:
\begin{equation}
\label{eq:diff_Frechet_SPD}
    \mls{DF}(P)\cdot V =
\begin{pmatrix}
\langle -2 \log_P(R_1), V \rangle_P
\\ \vdots \\
\langle  -2 \log_P(R_\ell), V \rangle_P
\end{pmatrix}
=
\begin{pmatrix} -2 \Tr(P^{-1/2}\log(P^{-1/2}R_1 P^{-1/2}) P^{-1/2} V)
\\ \vdots \\
-2 \Tr(P^{-1/2}\log(P^{-1/2}R_\ell P^{-1/2}) P^{-1/2} V)
\end{pmatrix}.
\end{equation}

\subsubsection{Local invertibility of $F$}

Based on \Cref{eq:diff_Frechet_SPD}, we investigate the local invertibility of the Fr\'echet map depending on the position of the set of reference points. We recall that $F$ is \textit{locally invertible} at a point $P \in \mls{SPD}(n)$ when $\mls{DF}(P)$ is of rank $m$. This also means that $F$ is a diffeomorphism from a certain neighborhood of $P$ to the image of that neighborhood. Viewing $(\log_P R_1, \ \ldots \ ,\log_P R_\ell)$ as a matrix in $\R^{m \times \ell}$ with $m=n(n+1)/2$ being the dimension of $\mls{Sym}(n)$, we introduce the set
\[
\Gamma_{R_1,\ldots,R_\ell}=\{P \in \mls{SPD}(n): \operatorname{rank}(\log_P R_1, \ \ldots \ ,\log_P R_\ell)=m\}.
\]

We note that $\Gamma_{R_1,\ldots,R_\ell}$ is an open subset of $\mls{SPD}(n)$, which is automatically empty for $\ell < m$. From \Cref{eq:diff_Frechet_SPD} it follows that $\Gamma_{R_1,\ldots,R_\ell}$ is precisely the set of points $P$ for which $\mls{DF}(P)$ is full rank. We can describe this set in an alternative way. To do so, let us denote by $S_{R_1,\ldots,R_\ell}$ the complement of $\Gamma_{R_1,\ldots,R_\ell}$, i.e., the set of all $P \in \mls{SPD}(n)$ such that $\mls{DF}(P)$ is rank deficient. For any $P \in S_{R_1,\ldots,R_\ell}$, one has that the tangent vectors $(\log_P(R_1),\ldots,\log_P(R_\ell))$ lie in a strict subspace of $\mls{Sym}(n)$. In other words, there exists $V \in \mls{Sym}(n)$, $V \neq 0$, such that the reference points $R_1,\ldots,R_\ell$ all belong to $\mathcal{H}_{P,V^\bot} = \{\exp_P(W): W \in \mls{Sym}(n), \, \Tr(VW)=0\}$. The set $\mathcal{H}_{P,V^\bot}$ can be thought of as one possible Riemannian equivalent of an affine hyperplane passing through $P$ and with tangent vectors orthogonal to $V$ at $P$. Similar to the Euclidean space, we see that $S_{R_1,\ldots,R_\ell}$ is the reunion of the ``Riemannian geodesic hyperplanes'' $\mathcal{H}_{P,V^\bot}$ for $P \in \mls{SPD}(n)$ and non-zero $V \in \mls{Sym}(n)$ that contain all the reference points. We obtain the following result:

\begin{proposition}
Any Fr\'{e}chet map $F$ on $\mls{SPD}(n)$ is locally invertible on
\begin{equation*}
\Gamma_{R_1,\ldots,R_\ell} = \mls{SPD}(n) \backslash \bigcup \left\{\mathcal{H}_{P,V^\bot}: \ R_1,\ldots,R_\ell \in \mathcal{H}_{P,V^\bot} \right\}.
\end{equation*}
\end{proposition}

Unfortunately, unlike for the Euclidean case, it is significantly more difficult to characterize which Riemannian hyperplanes contain a given set of reference points. Thus, we cannot give a more explicit geometric description of $\Gamma_{R_1,\ldots,R_\ell}$. Nevertheless, we obtain the following result.

\begin{theorem}
\label{thm:local_immersion}
Let $F$ be a Fr\'{e}chet map on $\mls{SPD}(n)$ with reference points $R_1,\ldots,R_\ell$ where $\ell \geq m$, where $m = \dim(\mls{SPD}(n))=n(n+1)/2$. If the reference points $R_1,\ldots,R_\ell$ are such that there exists $Q \in \mls{SPD}(n)$ for which $\{\log_Q R_1,\ldots, \log_Q R_\ell\}$ is of rank $m$, then the set $S_{R_1,\ldots,R_\ell}$ is of Lebesgue measure zero in $\mls{SPD}(n)$. In other words, $F$ is locally invertible on the open subset $\Gamma_{R_1,\ldots,R_\ell}$, which has full measure in $\mls{SPD}(n)$.
\end{theorem}

\begin{proof}
Let us first consider the case $\ell=m$. In that case, $S_{R_1,\ldots,R_m}$ is equivalently the set of all $P \in \mls{SPD}(n)$ such that $\det(\log_P R_1, \ \ldots \ ,\log_P R_m)=0$. We introduce the mapping $\psi:\mls{SPD}(n) \to \R$ defined by
\begin{equation*}
\psi(P) = \det(\log_P R_1, \ \ldots \ ,\log_P R_m).
\end{equation*}

Now, given the expression of the Riemannian logarithm \Cref{eq:aff_inv_exp_log}, since the matrix logarithm and square root are both real analytic functions on $\mls{SPD}(n)$ and the determinant of a $m \times m$ matrix is a polynomial function of its coefficients, it directly follows that $\psi$ is real analytic. Also, by assumption, $\psi$ is not identically zero (since $\psi(Q) \neq 0$) and thus, from standard results on zero sets of real analytic functions~\cite{mityagin2020zero}, the zero level set of $\psi$, in other words $S_{R_1,\ldots,R_m}$, is of vanishing Lebesgue measure in $\mls{SPD}(n)$.

For $\ell >m$, given some reference points $R_1,\ldots,R_\ell$, we see that $P \in S_{R_1,\ldots,R_\ell}$ if and only if all $m \times m$ subdeterminants of the matrix $(\log_P R_1, \ \ldots \ ,\log_P R_m)$ are equal to zero. With the assumption of the theorem, one of the $m \times m$ subdeterminants of $(\log_Q R_1, \ \ldots \ ,\log_Q R_m)$ is necessarily non-vanishing. Therefore, we can use the same argument as above and deduce that $S_{R_1,\ldots,R_\ell}$ is an intersection of sets at least one of which is of measure zero. Thus $S_{R_1,\ldots,R_\ell}$ is also of measure zero.
\end{proof}

We stress that the condition on the reference points in \Cref{thm:local_immersion} is very generic. For instance, with $Q = \Id$, it holds as soon as $(\log R_1, \ldots, \log R_\ell)$ is a full rank matrix. This is true for almost all choices of reference points (when considering the product Lebesgue measure of $\mls{Sym}(n)^\ell$). The proof also yields the fact that the set $S_{R_1,\ldots,R_\ell}$ is actually a real analytic subvariety of $\mls{Sym}(n)$. In the case where $\ell >m$, one could expect (as in the Euclidean case) that---for almost all choices of $\ell$ reference points---the Fr\'echet map is locally invertible on the whole space $\mls{SPD}(n)$. However, due to the lack of a simple characterization of geodesic hyperspaces in $\mls{SPD}(n)$, we are only able to obtain the statement under the following stronger assumption.

\begin{theorem}
\label{thm:local_immersion2}
Let $F$ be a Fr\'{e}chet map on $\mls{SPD}(n)$ with reference points $R_1,\ldots,R_\ell$ where $\ell \geq 2 m$, where $m = \dim(SPD(n))=n(n+1)/2$. For almost all choices of reference points in $\mls{SPD}(n)$, one has $S_{R_1,\ldots,R_\ell} = \emptyset$. In other words, the associated Fr\'{e}chet map is locally invertible on $\mls{SPD}(n)$.
\end{theorem}

\begin{proof}
Thanks to the above description of $S_{R_1,\ldots,R_\ell}$, we know that $S_{R_1,\ldots,R_\ell} \neq \emptyset$ if and only if there exist $P \in \mls{SPD}(n)$ and $V$ a unit Frobenius norm matrix of $\mls{Sym}(n)$ such that $R_1,\ldots,R_\ell \in \mathcal{H}_{P,V^\bot}$. Let us denote by $\overline{\mls{Symm}}(n)$ the unit sphere of $\mls{Sym}(n)$ and its tangent bundle $T\overline{\mls{Symm}}(n) = \{(V,W) \in \mls{Sym}(n) \times \mls{Sym}(n): \Tr(VW) = 0\}$ and by $T\overline{\mls{Symm}}(n)^{\otimes \ell}$ the $\ell$-times tangent bundle. We introduce the mapping
\begin{equation*}
\begin{aligned}
G: \mls{SPD}(n) \times T\overline{\mls{Symm}}(n)^{\otimes \ell} &\rightarrow \mls{SPD}(n)^\ell \\
(P,V,W_1,\ldots,W_\ell) &\mapsto (\exp_{P}(PW_1P),\ldots, \exp_{P}(PW_\ell P)).
\end{aligned}
\end{equation*}

We note that for any $P \in \mls{SPD}(n)$ and $(V,W_1,\ldots,W_\ell) \in T\overline{\mls{Symm}}(n)^{\otimes \ell}$, we have for all $i=1,\ldots,\ell$ that $\langle V,PW_iP \rangle_P = \Tr(P^{-1}V P^{-1}PW_iP)= \Tr(V W_i) = 0$ and thus the matrices $\exp_P(PW_1P),\ldots,\exp_P(PW_\ell P) \in \mathcal{H}_{P,V^\bot}$. Therefore, the set of all $\ell$-tuples of reference points for which $S_{R_1,\ldots,R_\ell} \neq \emptyset$ is contained in the image $G(\mls{SPD}(n) \times T\overline{\mls{Symm}}(n)^{\otimes \ell})$. As $G$ is a differentiable map and $\mls{SPD}(n) \times T\overline{\mls{Symm}}(n)^{\otimes \ell}$ is a manifold of dimension $2m-1+\ell m -\ell$, which is strictly smaller than $\ell m$ (the dimension of $\mls{SPD}(n)^\ell$) when $\ell\geq 2m$, we deduce that for almost all reference points $R_1,\ldots,R_\ell$, we have $S_{R_1,\ldots,R_\ell}=\emptyset$.
\end{proof}
An open question that we leave to future investigation is to determine whether the above still holds for a number of reference points $\ell$ between $m+1$ and $2m$, as it does for Fr\'{e}chet maps in Euclidean space.

\begin{remark}
    When considering the $1$-Fr\'{e}chet map $F^1$ instead of its squared version, the statement of \Cref{thm:local_immersion} still holds provided that there exists $Q \in \mls{SPD}(n)\backslash \{ R_1,\ldots R_\ell\}$ satisfying the theorem's condition.  Similarly, in \Cref{thm:local_immersion2}, the conclusion simply becomes that $F^1$ is a locally invertible map on $\mls{SPD}(n)\backslash \{ R_1,\ldots R_\ell\}$.
\end{remark}

\subsubsection{The case $\ell =m$}
A primary interest in this paper is to identify  Fr\'{e}chet maps $F$ with the smallest possible number of reference points in an attempt to limit the computational cost together with the dimension of the output space for $F$. We examine more closely the case $\ell =m$ and attempt to describe more precisely the real analytic subvariety $S_{R_1,\ldots,R_m}$ for certain specific configurations of reference points. When $\ell=m$, we recall that $P \in S_{R_1,\ldots,R_m}$ if and only if $\{\log_P R_1,\ldots,\log_P R_m\}$ are linearly dependent.

A first observation is that $S_{R_1,\ldots,R_m}$ can be, in fact, connected to the intrinsic notion of \textit{exponential barycentric subspace} that was introduced for general Riemannian manifolds and studied in~\cite{pennec2018barycentric}. The exponential barycentric subspace of $R_1,\ldots,R_m$ is defined by
\begin{equation*}
    \mls{EBS}(R_1,\ldots,R_m) =
    \left\{P: \ \exists \lambda \in \mathcal{P}_m^*, P\,\,\text{is a critical point of}\,\,\sigma(Q;\lambda)^2 \doteq \sum_{i=1}^{m} \lambda_i d(R_i,Q)^2 \right\},
\end{equation*}

\noindent where $\mathcal{P}_m^*$ denotes the set of $\lambda=(\lambda_i) \in \R^m$ such that $\sum_{i=1}^m \lambda_i =1$. This can be interpreted as the set of all the weighted Fr\'{e}chet barycenters of the $R_i$'s. Using again \Cref{lemma:diff_Frechet}, the criticality condition is equivalent to $\sum_{i=1}^{m} \lambda_i \log_P R_i = 0$. The latter implies in particular that $\{\log_P R_1,\ldots,\log_P R_m\}$ are linearly dependent. We deduce that
\begin{equation*}
    \mls{EBS}(R_1,\ldots,R_m) \subseteq S_{R_1,\ldots,R_m}.
\end{equation*}

\noindent In contrast to the Euclidean case, these two sets are not necessarily equal as $S_{R_1,\ldots,R_m}$ also contains all the critical points of the functions $\sigma(Q;\lambda)^2$, i.e., the solutions of $\sum_{i=1}^{m} \lambda_i \log_P R_i = 0$ for the non-zero $\lambda \in \R^m$ with $\sum_{i=1}^m \lambda_i = 0$.

Although~\cite{pennec2018barycentric} provides more explicit characterizations of $\mls{EBS}$ in constant curvature manifolds (such as the sphere and hyperbolic plane), there is unfortunately no known corresponding simple description in the case of $\mls{SPD}(n)$. We conjecture that a possible obstruction lies in the fact that $\mls{EBS}(R_1,\ldots,R_m)$ or $S_{R_1,\ldots,R_m}$, unlike affine hyperplanes in a Euclidean space, are not necessarily totally geodesic. In what follows, we say that a submanifold $S\subset \mls{SPD}(n)$ is \textit{totally geodesic} if it is geodesically complete (i.e., $\exp_P(V) \in S$ for any $P \in S$ and $V \in T_P S$) and if any geodesic of $S$ is also a geodesic in $\mls{SPD}(n)$. In symmetric spaces such as $\mls{SPD}(n)$, totally geodesic submanifolds can be characterized via the Lie triple system condition (see Theorem 7.2 in~\cite{helgason1979differential} or~\cite{tumpach2024totally}). This condition allows us to fully classify all totally geodesic submanifolds of codimension {1} in $\mls{SPD}(n)$, as stated below.

\begin{proposition}
\label{prop:tot_geod_SPD}
For $n\geq 3$, the totally geodesic submanifolds of $\mls{SPD}(n)$ of dimension $m-1$ are exactly the subsets $\mls{SPD}_r(n) = \{P\in \mls{SPD}(n): \det(P)=r\}$ for $r>0$. Furthermore, there is a well-defined projection $\Pi_r:\mls{SPD}(n) \to \mls{SPD}_r(n)$ given by $\Pi_r(P) = \sqrt[n]{r/\det(P)}\, P$ for all $P \in \mls{SPD}(n)$.
\end{proposition}

It is easily verified that $\mls{SPD}_r(n)$ is totally geodesic based on the expression of geodesics \Cref{eq:SPD_geod}. Showing that these are, in fact, the only ones based on the Lie triple system condition is a little more involved. Since, to our knowledge, there is no statement of this result in the literature, a proof is detailed in \Cref{app:proof_tot_geod}, in which we also derive a corresponding result for the special case of $n=2$.

Based on \Cref{prop:tot_geod_SPD}, let us now consider the situation in which all $m$ reference points belong to a totally geodesic hypersurface $\mls{SPD}_r(n)$ for some $r>0$. We then ask the question whether $\mls{EBS}(R_1,\ldots,R_m)$ or $S_{R_1,\ldots,R_m}$ precisely coincide with $\mls{SPD}_r(n)$ in this case. A partial answer is given by the following result.

\begin{theorem}
\label{thm:EBS_tot_geod}
Assume that $R_1,\ldots,R_m \in \mls{SPD}_r(n)$ for some $r>0$. Then
\begin{equation*}
\mls{EBS}(R_1,\ldots,R_m) \subseteq \mls{SPD}_r(n) \subseteq S_{R_1,\ldots ,R_m}.
\end{equation*}

\noindent Moreover, these three spaces are all equal unless the reference points lie on a submanifold of the form $\exp_P(H)$ for $P \in \mls{SPD}_r(n)$ and $H$ an affine subspace of $\mls{Sym}(n)$ of dimension at most $m-2$. Under this condition, the Fr\'{e}chet map associated to $R_1,\ldots,R_m$ is an immersion on each of the halfspaces $\mls{SPD}^{+}_r(n) = \{P: \det(P)>r\}$ and $\mls{SPD}^{-}_r(n) = \{P: \det(P)<r\}$.
\end{theorem}

The proof can be found in \Cref{app:proof_thm_EBS_tot_geod}. We stress that one question that remains unaddressed in \Cref{thm:EBS_tot_geod} is whether the condition of equality between the exponential barycentric subspace $S_{R_1,\ldots,R_m}$ and $\mls{SPD}_r(n)$ is truly generic. Indeed, while it can be interpreted as a form of affine independence of the reference points within the submanifold $\mls{SPD}_r(n)$, it still involves checking all possible foot points $P\in \mls{SPD}_r(n)$, unlike the notion of affine independence that is introduced in~\cite{pennec2018barycentric}. We leave it to future investigations to establish if, for instance, almost all configurations of $m$ reference points in $\mls{SPD}_r(n)$ do satisfy this condition. This issue notwithstanding, the result of \Cref{thm:EBS_tot_geod} provides a picture in part reminiscent of the Euclidean case: choosing reference points located on a totally geodesic hypersurface $\mls{SPD}_r(n)$ results in a Fr\'{e}chet map that is a local diffeomorphism on each of the two halfspaces delimited by $\mls{SPD}_r(n)$.

\begin{remark}
We emphasize that we have mainly focused the discussion on the local injectivity of the Fr\'{e}chet map on $\mls{SPD}(n)$. Some natural follow-up questions are to determine which configurations of reference points further result in $F$ being a globally injective map, what is then the image of $F$, and how to obtain the inverse map, as in the Euclidean situation discussed in \Cref{s:frechet_euclidean}. To our knowledge, these questions become significantly more difficult in $\mls{SPD}(n)$. Our preliminary exploration of the simple case of $\ell=3$ reference points in the three-dimensional $\mls{SPD}(2)$ manifold has shown that finding the inverse of $F$, i.e., reconstructing a matrix $P \in \mls{SPD}(2)$ from its distances to the three reference points, can be reduced to solving a single analytic equation on the determinant of $P$. However, the precise structure of solutions of this equation, even in this basic case, remains elusive. Numerical evidence suggests a quite different picture from the Euclidean case given by \Cref{thm:invertibility_Frechet_Eucl}. For these reasons, we leave such issues for future investigation.
\end{remark}

\subsection{$k$-means clustering methods in SPD}\label{s:clustering_spd}

In this section, we describe several possible strategies for applying $k$-means on datasets of SPD matrices, which will be compared numerically in the next section.

The first and most straightforward approach, which we henceforth refer to as the \emph{Intrinsic Riemannian Clustering} ({\bf IRC}), is the direct transposition of the classical Lloyd's algorithm, reported in \Cref{s:background} as \Cref{algo:lloyd}, to the SPD manifold setting. This transposition requires replacing the Euclidean distance with the Riemannian distance \Cref{eq:SPD_distance} on SPD and computing the cluster centroids in step 5 \Cref{algo:lloyd} with the Fr\'{e}chet mean, rather than the arithmetic mean. For a given set $\{X_1,\ldots,X_N\}\subset \mls{SPD}(n)$, computing the Fr\'{e}chet mean requires solving the minimization problem
\begin{equation}
\label{eq:min_problem_Fr_mean}
   \overline P = \argmin_{P\in \mls{SPD}(n)} \left\{  f(P) = \frac{1}{N}\sum_{i=1}^{N} d(X_i,P)^2 \right\}.
\end{equation}

Due to the aforementioned properties of SPD, this is a convex problem with a unique solution. It can be solved numerically via Riemannian gradient descent in $\mls{SPD}(n)$ with gradient $g(P) = - (2/N)\sum_{i=1}^{N} \log_{P} X_i$. The gradient can be directly expressed as described in \Cref{s:frechet_spd_properties}. We adopt the standard scheme from~\cite{afsari2013convergence} which is summarized in \Cref{fmean.gd} below.

\begin{algorithm}
\caption{Gradient descent method to estimate the Fr\'{e}chet mean (based on~\cite{afsari2013convergence}).}\label{fmean.gd}
\begin{algorithmic}[1]
\STATE {\bf Input: } A set of points $X_1,\dots, X_N \in \mls{SPD}(n)$ and an initialization $P_0 \in \mls{SPD}(n)$
\STATE $P^{(0)} \gets P_0$, $t \gets 0$, stop $\gets$ false
\WHILE{$\neg$ stop}
\STATE $g(P^{(t)}) \gets - (2/N)\sum_{i=1}^{N} \log_{P^{(t)}} X_i$
\STATE $\eta^{(t)} \gets$ compute step size (constant or adaptative)
\STATE $P^{(t+1)} \gets \exp_{P^{(t)}}\left(-\eta^{(t)}\,g(P^{(t)})\right)$
\STATE $t \gets t+1$
\STATE stop $\gets$ check convergence
\ENDWHILE
\STATE {\bf Output:} Fr\'{e}chet mean $\overline P =P^{(t)}$
\end{algorithmic}
\end{algorithm}

Since repeated evaluations of Fr\'{e}chet means is the main computational bottleneck of the $k$-means method in manifolds, some algorithmic strategies have been proposed to reduce computational cost, most notably the \emph{Iterative Centroid Method} ({\bf ICM}) proposed in~\cite{cury2013template, ho2013recursive}, which is based on the \emph{recursive barycenter} scheme suggested in \cite{sturm2003probability}. This approach approximates the Fr\'{e}chet mean of a set $X_1, \dots, X_N \in \mls{SPD}(n)$ from the iterations
\[
\begin{aligned}
P_1 &= X_1,\\
P^{(t+1)} &= \big(P^{(t)}\big)^{1/2}\big(\big(P^{(t)}\big)^{-1/2}X^{(t+1)}\big(P^{(t)}\big)^{-1/2}\big)^{1/(t+1)}\big(P^{(t)}\big)^{1/2}
\end{aligned}
\]

\noindent for $t=1,\ldots,N-1$. In other words, starting from the first element $X_1$ in the set, one keeps moving the mean estimate along the geodesic connecting it to the next element with a step decreasing as the inverse of the iteration number. Thus, this approach involves computing $N$ geodesics, which is comparable to a single iteration of the full gradient descent method of \Cref{fmean.gd} (albeit one can parallelize the latter). Note that this scheme is also related to a form of stochastic gradient descent method for \Cref{eq:min_problem_Fr_mean}. In the following, we refer to this specific $k$-means algorithm, which uses this recursive scheme to estimate the Fr\'{e}chet mean at each iteration, as \emph{Approximate Riemannian Clustering} ({\bf ARC}).

As discussed above, in this work, we consider an alternative approach for performing $k$-means on manifold data that involves a prior embedding of the data into some Euclidean space. Our proposed approach, called \emph{Fr\'{e}chet Map Clustering}  ({\bf FMC}), uses a Fr\'{e}chet map and is summarized in \Cref{alg:FMC}. As discussed above, the properties of the Fr\'{e}chet map and, hence, the properties of the algorithm, are fundamentally linked to the number and selection of the reference points. We discuss and evaluate different strategies for reference point selection in \Cref{s:results}.

\begin{algorithm}
\caption{Fr\'{e}chet Map Clustering (FMC)}
\label{alg:FMC}
\begin{algorithmic}[1]
\STATE {\bf Input}: A dataset $\D = \{X_1, \dots , X_N\} \subset \mls{SPD}(n)$, the number of clusters $k \in \N$.
\STATE {\bf Parameters:} A set $\{R_1,\ldots,R_\ell\}$ of reference points in $\mls{SPD}(n)$, order $p=1,2$.
\STATE Define the Fr\'{e}chet map $F^p : P \mapsto F(P) = (d(P,R_1)^p, \ldots,d(P,R_\ell)^p)\in \mathbb{R}^\ell$ and compute the image of $\D$ by $F^p$: $F^p(\D) \gets \{F^p(X_1), \ldots,F^p(X_N)\} \subset \mathbb{R}^\ell$.
\STATE Apply the $k$-means algorithm in $\R^\ell$ to the set $F^p(\D)$ to partition the dataset into $k$ clusters
$H_1, \dots, H_k \subset \R^\ell$
\STATE Identify the corresponding clustering $(\mls{CL}_1,\dots,\mls{CL}_k)$ in $\mls{SPD}(n)$ by a re-labeling the data points which defines $\mls{CL}_r = \{X_j \in \mathcal{D}\;|\; F^p(X_j) \in H_r\}$.
\STATE {\bf Output}: A partition of $\D$ into the clusters $(\mls{CL}_1, \dots,\mls{CL}_k)$.
\end{algorithmic}
\end{algorithm}

Since computing means and distances in the Euclidean setting is very fast compared to the analog manifold computations, the critical advantage of the FMC algorithm over the methods described above is to carry out the iterative steps of $k$-means in a Euclidean space rather than a Riemannian manifold $\M$. This is illustrated in \Cref{cfe.fig}. The computationally most expensive step in \Cref{alg:FMC} is step 3, which requires computing $\ell\,r$ distances $d(X,R_i)$ in $\mls{SPD}(n)$. Each such distance computation requires at most $\mathcal{O}(n^3)$ floating point operations, so that the total cost of step 3 is $\mathcal{O}(\ell\, N \, n^3)$, where $N$ is the number of data points and $\ell$ is the size of the reference set in $\mls{SPD}(n)$.

Note that the injectivity of the Fr\'echet mapping $F$ is not necessary to implement FCM since step 5 is simply a re-labeling of the original data points.

\begin{figure}
\centering
\includegraphics[width=0.6\textwidth]{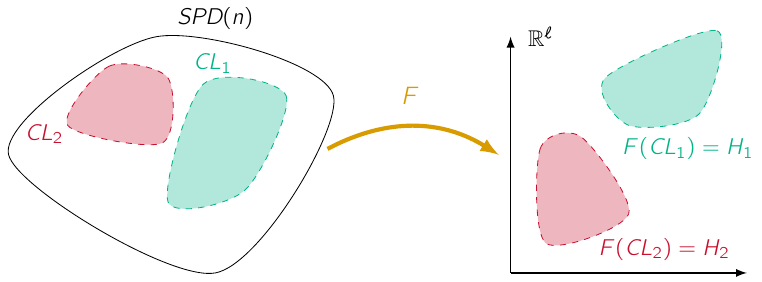}
\caption{FMC algorithm. The Fr\'echet map $F$ takes a finite set $\mathcal{D} \subset \mls{SPD}(n)$ into $\mathbb{R}^\ell$. Next, the $k$-means algorithm in $\mathbb{R}^\ell$ is applied to partition the set $F(\mathcal{D})$ into $k$ clusters $\{H_1,\dots,H_k\}$. Finally, a simple re-labeling of the data is to identify the corresponding clusters $(\mls{CL}_1,\ldots,\mls{CL}_k)$  of $\mathcal{D}$ in $\mls{SPD}(n)$.}\label{cfe.fig}
\end{figure}

For comparison, we also consider below the \emph{Log-Euclidean Clustering} ({\bf LEC}), originally proposed in \cite{arsigny2007geometric}, which also maps the data from $SPD(n)$ into a Euclidean space. LEC follows steps similar to \Cref{alg:FMC}, the main difference being that the Fr\'{e}chet map $F$ is replaced with the log map $P \mapsto \log_R P \in \mls{Sym}(n)$ for a given reference point $R$ (typically $R = \Id$). Subsequently, the $k$-means algorithm is applied in the Euclidean space $\mls{Sym}(n) \approx \mathbb{R}^{n(n+1)/2}$ and the cluster points are then mapped back into $\mls{SPD}(n)$.

\subsection{Reference point selection}\label{s:ref_selection}

We found that the performance of the FMC algorithm (\Cref{alg:FMC}) depends significantly on the positioning of the reference points $\{R_i\}_{i=1}^\ell$, $\ell \in \mathbb{N}$, making the selection of reference points a critical component for deployment. Our goal is a strategy that is computationally fast, cluster-agnostic, and consistently ensures high clustering accuracy. While the theoretical analysis of the previous sections gives a certain insight into how the choice of reference points influence the properties of the resulting Fr\'{e}chet map, these results were mostly focused on the situation of a number of reference points $\ell$ equal or larger than the manifold's dimension $m$. Here, we will instead introduce heuristic approaches in which $\ell$ is chosen sometimes much smaller than $m$. Although this a priori breaks the invertibility of the Fr\'{e}chet map, we empirically found that the overall clusters structure of the data can still be preserved sufficiently and lead to robust clustering results for the FMC method. We will specifically investigate two heuristics: a random selection approach and a principled strategy that requires tuning two scalar hyperparameters. \Cref{s:performance_eval} details our numerical validation of these strategies.

\subsubsection{Random reference point selection}\label{s:rand_ref_selection}

The simplest method of selecting reference points for \Cref{alg:FMC} is random selection. While placing them randomly on a hypersphere enclosing the dataset was explored, it yielded suboptimal results due to our general lack of knowledge about the dataset's geometry. Therefore, we adopted a more practical strategy in which we randomly select the reference points directly from the dataset itself.

\subsubsection{Reference point selection based on approximate Fr\'{e}chet means}\label{s:principled_ref_selection}

Although using random point selection is appealing due to its simplicity, it might lead to inconsistent clustering results. Consequently, we designed a more principled approach that takes into account the geometry of our problem. In summary, our idea is to place the reference points $\{R_m\}_{m=1}^\ell$ on the geodesic $\gamma_{M_iM_j}(t)$ that runs through the centroids, i.e., the Fr\'echet means $M_i$ and $M_j$, of each pair of clusters $\mls{CL}_i$ and $\mls{CL}_j$. Unfortunately, this requires knowing the clusters (true labels), which obviously defeats the purpose. However, since we are solely interested in placing reference points, it is sufficient to compute an approximate solution of the clustering problem so that we can compute approximate Fr\'echet means $\tilde{M}_i$, $\tilde{M}_j$ and a geodesic $\gamma_{\tilde{M}_i \tilde{M}_j}(t)$ running through those points, where
\begin{equation}\label{e:frechet_mean_geo}
\gamma_{\tilde{M}_i \tilde{M}_j}(t) \defeq \tilde{M}_i^{1/2}\exp\!\big(t \tilde{M}_i^{-1/2} \tilde{M}_j \tilde{M}_i^{-1/2}\big) \tilde{M}_i^{1/2}, \quad t\in\R.
\end{equation}

The next tasks we face are where to place the reference points along this geodesic and how many reference points per cluster pair $\mls{CL}_i$ and $\mls{CL}_j$ to select to accurately cluster the data. We elaborated the following strategy on the basis of empirical observations.
\begin{enumerate}
\item For each pair of approximate Fr\'echet means $\tilde{M}_i$, $\tilde{M}_j$, we place \emph{two} reference points along the geodesic connecting $\tilde{M}_i$ and $\tilde{M}_j$.
\item We place the reference points \emph{between} $\tilde{M}_i$ and $\tilde{M}_j$ if the clusters are ``far'' from each other (``far case''); we place the reference points \emph{outside} the geodesic curve connecting $\tilde{M}_i$ and $\tilde{M}_j$, if the clusters are ``close'' (``close case''). This is shown in~\Cref{f:far_close_case}.
\end{enumerate}

\begin{figure}
\centering
\includegraphics[height=3cm]{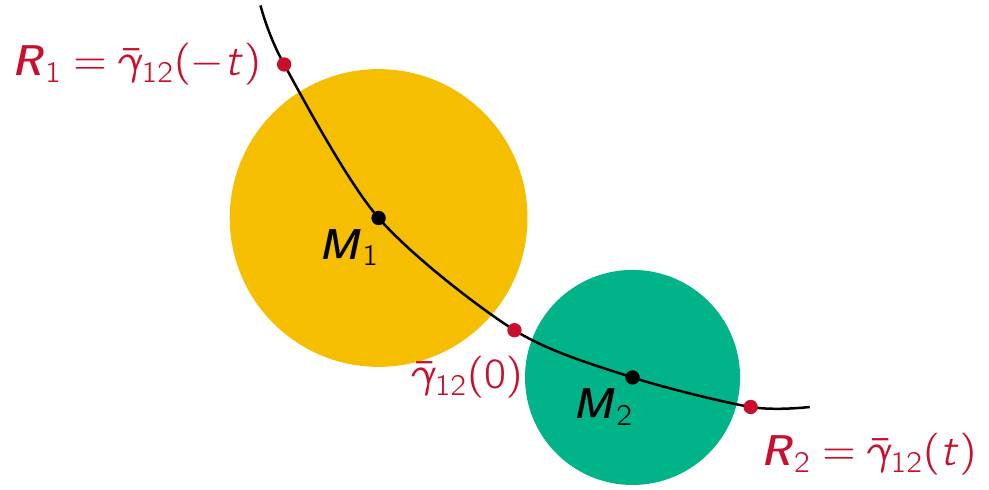}
\includegraphics[height=3.8cm]{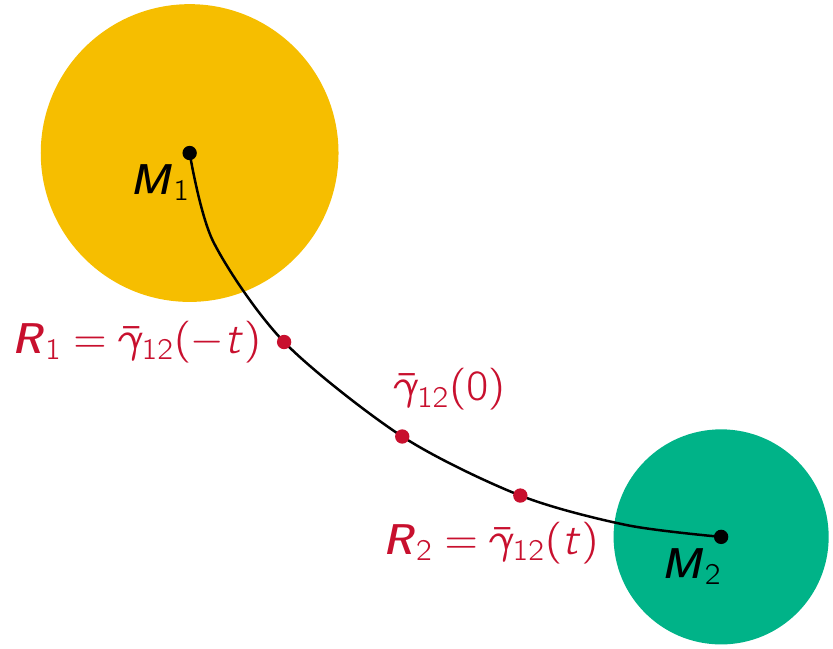}
\caption{Illustration of the proposed strategy to select reference points for the Fr\'echet map $F$. Left: Close case. We select the reference points $R_1$ and $R_2$ outside the segment that connects the Fr\'echet means $M_1$ and $M_2$ of the clusters $\mls{CL}_1$ and $\mls{CL}_2$. Right: Far case. We select the reference points $R_1$ and $R_2$ inside the segment that connects the centers $M_1$ and $M_2$}\label{f:far_close_case}
\end{figure}

This strategy requires us not only to provide an estimate for the clusters $\{\mls{CL}_i\}_{i=1}^k$ but also to estimate the radius $\rho_i > 0$ of each cluster $\mls{CL}_i$. Moreover, we need to quantify what we mean by the cluster pair being ``far'' or ``close'' and determine where along the geodesic to place the two reference points. The following remarks explain the rationale of our strategy.

\begin{remark}
In an attempt to embed the data into a Euclidean space with the lowest possible dimension, we initially selected only one reference point per pair $\{\mls{CL}_i, \mls{CL}_j\}$. However, we observed heuristically that the clustering performance of the FMC algorithm could be unstable with respect to permutations of the location of this reference point (e.g., depending on the reference point being located on the far side of cluster $\mls{CL}_i$ or $\mls{CL}_j$). If the number of clusters is small, doing permutations is a viable strategy since computing the permutations is not expensive, and selecting the best reference point assignment based on, e.g., dispersion or some criterion related to the downstream decisions based on the clustering performance, does not add significant runtime. However, if the number of clusters increases, this permutation approach becomes expensive. We observed that selecting two reference points per pair $\{\mls{CL}_i, \mls{CL}_j\}$ makes the FMC algorithm significantly more stable. A potential downside is that this approach embeds the data in a higher-dimensional space.
\end{remark}

\begin{remark}
We introduced the ``far'' and ``close'' cases since we observed that computing long geodesics $\gamma(t)$ can introduce numerical instability. To avoid computing long geodesics, we place the reference points between the clusters if the approximate Fr\'echet means $\tilde{M}_i$ and $\tilde{M}_j$ are ``far'' from each other relative to the Riemannian radii $\rho_i$ and $\rho_j$ of the clusters.
To differentiate these two cases and, thus, determine the position of the reference points, we assign the coordinate along the geodesic by choosing
\begin{equation}\label{e:cases}
t_{ij} =
\begin{cases}
t_{\text{close}} & \text{if }\,\,\dfrac{d(\tilde{M}_i,\tilde{M}_j)}{\tilde{\rho}_i+\tilde{\rho}_j} < \epsilon_d, \,\,i,j=1,\ldots,k,\,\, i \not= j, \\
t_{\text{far}} & \text{otherwise},
\end{cases}
\end{equation}
 for the ``close case'' and the ``far case,'' respectively. See \cref{f:far_close_case} for an illustration. Here, $\tilde{\rho}_i$ and $\tilde{\rho}_j$ represent the approximate Riemannian radii of the convex hull of $\widetilde{\mls{CL}}_i$ and $\widetilde{\mls{CL}}_j$, respectively, and $\epsilon_d > 0$, $t_{\text{close}}>1$, and $0 \leq t_{\text{far}} < 1$ are user defined parameters. The parameter $t_{ij}$ controls how far along the geodesic running through $\tilde{M}_i$ and $\tilde{M}_j$ we place the reference points.
\end{remark}

Next, we address the question of how to approximate $\mls{CL}_i$, $M_i$, and $\rho_i$. While the simplest approach is to carry out the approximation directly in Euclidean space (using the embedding of $SPD(n)$ in $\R^m$), this tends to lead to inconsistent clustering results. A better approximation is obtained using the log-Euclidean framework, in which we compute the approximate clusters $\{\widetilde{\mls{CL}}_i\}_{i=1}^k$ by applying $k$-means in Euclidean space after embedding the data $\mathcal{D}$ into the tangent space using matrix logarithms $\log_{\Id}:\mls{SPD}(n)\to\R^{n(n+1)/2}$ (we evaluate the matrix logarithm at $\Id$ in all our experiments). Subsequently, we use the clusters $\widetilde{\mls{CL}}_i$ to compute the Fr\'echet means $\tilde{M}_i$ on the manifold using ICM~\cite{cury2013template, ho2013recursive} (see also \Cref{s:clustering_spd}). To find $\tilde{\rho}_i$ we draw $N_\rho \ll N$ samples $X_j$ from $\widetilde{\mls{CL}}_i \subset \mathcal{D} = \{X_1,\ldots,X_N\}$ for each approximate cluster $\widetilde{\mls{CL}}_i$. Given $\tilde{M}_i$ and $\tilde{\rho}_i$, we determine the location of the reference points based on the criterion~\Cref{e:cases}. To place the reference points, we compute the midpoint $\bar{M}_{ij} = \gamma_{\tilde{M}_i\tilde{M}_j}(1/2)$ of the geodesic that connects $\tilde{M}_i$ and $\tilde{M}_j$ using \Cref{e:frechet_mean_geo}. Given $\bar{M}_{ij}$, we position the reference points along the geodesic at locations $\bar{\gamma}_{ij}(\pm t_{ij})$ based on the auxiliary geodesic
\begin{equation}\label{e:aux_geo}
\bar{\gamma}_{ij}(t) \defeq \bar{M}_{ij}^{1/2}\exp\!\big(t\bar{M}_{ij}^{-1/2}M_j\bar{M}_{ij}^{-1/2}\big)\bar{M}_{ij}^{1/2}
\end{equation}

\noindent defined by the midpoint $\bar{M}_{ij}$. We summarize this approach in \Cref{a:refpoint-selection}.

\begin{remark}
Using the log-Euclidean framework to identify approximate clusters $\{\widetilde{\mls{CL}}_i\}_{i=1}^k$ as a way to select reference points $\{R_m\}_{m=1}^\ell$ may seem counterintuitive, since, in some cases, this approach already solves the clustering problem. However, we show in \Cref{s:results} that our framework yields more accurate and more consistent results than the log-Euclidean framework
at a comparable computing cost.
\end{remark}

By the description above, the selection of the reference point for our FCM algorithm requires assigning the values of the hyperparameters: $t_{\text{close}}$, $t_{\text{far}}$, $N_\rho$, and $\epsilon_d$. We discuss the sensitivity of our FCM algorithm with respect to changes in these hyperparameters in \Cref{s:results}.

\begin{algorithm}
\caption{Reference point selection based on approximate Fr\'{e}chet means. We use the Fr\'echet means of approximate clusters $\{\widetilde{\mls{CL}}_i\}_{i=1}^k$ found via the log-Euclidean framework to define an auxiliary geodesic $\bar{\gamma}_{ij}$ that is used to position two reference points per cluster pair. }\label{a:refpoint-selection}
\begin{algorithmic}[1]
\STATE {\bf Input: } A set of points $X_1,\dots, X_N \in \mls{SPD}(n)$, the number of clusters $k$, and hyperparameters $t_{\text{close}}$, $t_{\text{far}}$, $N_\rho$, and $\epsilon_d$
\STATE $\mathcal{E} \gets \{\log_{\Id} (X_1),\dots, \log_{\Id} (X_N) \}$
\STATE $\{\widetilde{\mls{CL}}_i\}_{i=1}^k \gets$ apply Euclidean $k$-means to $\mathcal{E}$
\STATE $\{\tilde{M}_i\}_{i=1}^k \gets$ compute Fr\'echet means on $\mls{SPD}(n)$ for each cluster $\widetilde{\mls{CL}}_i$ using ICM
\STATE $\{\tilde{\rho}_i\}_{i=1}^k \gets$ estimate Riemannian radius $\tilde{\rho}_i$ for each $\widetilde{\mls{CL}}_i$ by drawing $N_{\rho}$ random samples $X_j$ within $\widetilde{\mls{CL}}_i$; compute $90\%$ quantile $\rho_i$ of the distances $d(\tilde{M}_i,X_j)$
\STATE $\{t_{ij}\} \gets$ given $\tilde{M}_i$, $\tilde{M}_j$, $\tilde{\rho}_i$, $\tilde{\rho}_j$, determine $t_{ij}$ based on \Cref{e:cases} for each $i,j=1,\ldots,k$, $i \not = j$
\STATE $\{\bar{M}_{ij}\} \gets$ compute midpoints \[\bar{M}_{ij} \defeq \gamma_{\tilde{M}_i\tilde{M}_j}(1/2)\] for each pair of distinct clusters $\{\widetilde{\mls{CL}}_i, \widetilde{\mls{CL}}_j\}$, $i,j=1,\ldots,k$, $i \not = j$, via \Cref{e:frechet_mean_geo}
\STATE $\{R_m\}_{m=1}^\ell \gets $ given $t_{ij}$ and $\{\bar{M}_{ij}\}$ compute two reference points (in symmetric location) based on the auxiliary geodesic $\bar{\gamma}_{ij}$ \Cref{e:aux_geo} evaluated at $\pm t_{ij}$
\STATE {\bf Output:} Reference points $\{R_m\}_{m=1}^\ell \subset \mls{SPD}(n)$
\end{algorithmic}
\end{algorithm}

\subsection{Performance Evaluation}\label{s:performance_eval}

In order to assess and compare clustering methods, we report different performance scores that we briefly recap below. The first criterion is the overlap between the estimated clusters and some \emph{ground truth} labels, assuming, in the context of method validation, that these are known. For classification tasks, accuracy is a key metric to evaluate model performance. One popular approach to report accuracy is based on the so-called confusion matrix, which compares true labels with predicted labels. In clustering scenarios, since there is no \emph{a priori} correspondence between the cluster indices and the original labels, one needs to first solve an optimal assignment problem to find the best permutation of the estimated clusters that match them to \emph{ground truth} groups. This is classically done using, e.g., the \emph{Hungarian method}~\cite{kuhn1955hungarian}. We apply this exact approach to obtain correspondences between clusters and \emph{ground truth} labels from which we can then compute the confusion matrix between the two sets of clusters. We then define the clustering accuracy as the proportion of data samples assigned to the correct group.

A second quantitative measure of clustering is the total Riemannian dispersion, which is the objective function of the intrinsic Riemannian $k$-means method presented in \Cref{s:clustering_spd}. If $\widehat{\P} = (\widehat{\mls{CL}}_1,\ldots,\widehat{\mls{CL}}_k)$  are the output clusters for any given method, the Riemannian dispersion is given by \Cref{eq:totaldisp}, where $d_\mathcal{M}$ is the affine-invariant Riemannian distance on $\mls{SPD}(n)$. To provide an easier-to-interpret score, whenever some \emph{ground truth} groups $\P=(\mls{CL}_1,\ldots,\mls{CL}_k)$ are known for the considered example, we also report the normalized dispersion:
\begin{equation}
\label{eq:normalized_totdisp}
    \overline{\operatorname{totdisp}}(\widehat{\P}) = \frac{\operatorname{totdisp}(\widehat{\P})}{\operatorname{totdisp}(\P)}.
\end{equation}
We note that $\P$ may not necessarily correspond to the minimal total dispersion among all partitions in general. Consequently, the above normalized total dispersion may, in some cases, be smaller than one. As such, this should be understood as a mere renormalization to make the obtained values of the dispersion more interpretable across different simulations.

\section{Numerical Experiments}\label{s:results}
In this section, we present extensive numerical experiments using both synthetic and real-world data in $\mls{SPD}(n)$ to illustrate our FMC algorithm.

\subsection{Reference Point Selection}\label{s:res_refpoints}
Our first set of experiments aimed to validate the strategies for the selection of reference points in the FMC algorithm we presented in \Cref{s:ref_selection} using synthetic data.

For these experiments, we generated multiple balls $B_i$, $i=1,\dots,k$, in $\mls{SPD}(4)$ by randomly sampling their centers $C_i$ and assigning corresponding radii $\rho_i$ drawn uniformly from $[0.8,1.2]$. The centers are retained only if their normalized pairwise distances satisfy
\begin{equation}\label{eq:distratio}
    d_{\text{low}} \leq \frac{d(C_i, C_j)}{\rho_i + \rho_j} \leq d_{\text{up}}, \quad 1 \leq i \neq j \leq k,
\end{equation}
for some user-defined parameters $ d_{\text{low}}, d_{\text{up}}$. As normalized pairwise distances close to $1$ make the balls harder to separate, in our experiments below, we selected these parameters close to 1. For each ball, we generated 4\,000 samples uniformly distributed within their boundary.

\subsubsection{Random reference point selection}

\textbf{Purpose:} To assess the performance of the FCM algorithm using the random reference point selection strategy presented in \Cref{s:rand_ref_selection}, for various choices of the number of clusters $k$.

\textbf{Setup:} We generated multiple benchmark synthetic data, as described above, with the number of balls $k \in \{2,3,4,5\}$. For each value of $k$, we generated 100 instances using a varying number of reference points
$\ell$ chosen randomly within the dataset.

\textbf{Results:} We report in \Cref{t:fcf_random_detailed_results} the result of the FCM algorithm. The table includes the mean accuracy (averaged over the 100 runs), its standard deviation, the minimum accuracy (i.e., the worst performance observed), and the 10th percentile (the accuracy threshold below which the lowest 10\% of runs fall).

\begin{table}
\caption{FCM results for a synthetic benchmark dataset of $k$ balls with random centers $C_i$ and random radii $\rho_i$. For each choice of $\ell$ (number of reference points), we generate 100 benchmark sets of disjoint balls. We report the average clustering accuracy for varying $\ell$ and $k$ (ground truth data is available).}
\label{t:fcf_random_detailed_results}
\tabadjust
\begin{tabular}{cccccc}
\toprule
$k$ & $\ell$ & mean (\%) & std & min (\%) & 10th (\%) \\
\midrule
2 &  2 & 99.28\% & 0.010 & 96.00\% & 98.00\% \\
  &  4 & 99.83\% & 0.004 & 98.00\% & 99.00\% \\
  &  6 & 99.92\% & 0.003 & 98.00\% & 100.00\% \\
3 &  3 & 94.54\% & 0.090 & 68.00\% & 76.70\% \\
  &  6 & 98.46\% & 0.040 & 74.00\% & 97.90\% \\
  &  9 & 99.74\% & 0.006 & 97.00\% & 99.00\% \\
4 &  6 & 94.57\% & 0.070 & 73.00\% & 81.70\% \\
  & 12 & 98.96\% & 0.015 & 91.00\% & 97.90\% \\
  & 18 & 99.69\% & 0.005 & 98.00\% & 99.00\% \\
5 & 10 & 95.26\% & 0.040 & 79.00\% & 88.70\% \\
  & 20 & 98.48\% & 0.012 & 95.00\% & 97.00\% \\
  & 30 & 98.95\% & 0.007 & 96.00\% & 98.00\% \\
\bottomrule
\end{tabular}
\end{table}

\textbf{Conclusion:} For a fixed $k$, increasing $\ell$ steadily improves the mean accuracy. Random reference point selection does not guarantee very high clustering accuracy, especially as the manifold dimension increases. Specifically: (1) for fixed $k$, increasing $\ell$ steadily improves mean accuracy and reduces variability (higher minima and 10th percentiles); (2) as $k$ increases, a larger $\ell$ is required to achieve near-perfect accuracy;
(3) for $k=5$, the performance plateaus below $99\%$ even at $\ell=30$, indicating diminishing returns and a clear gap relative to $k\le 4$.

\subsubsection{Reference point selection based on approximate Fr\'{e}chet means}

\textbf{Purpose:} To assess the performance of the FMC algorithm using the reference point selection approach in~\Cref{s:principled_ref_selection} for various choices of the hyperparameters $t_{\text{close}}$ and $t_{\text{far}}$.

\textbf{Setup 1:} We created a benchmark synthetic dataset using $k=5$ balls in $\mls{SPD}(4)$ and generated 100 random configurations  with $d_{\text{low}} = 1.1$ and $d_{\text{up}} = 3$. Following \Cref{e:cases} with $\epsilon_d = 2.5$  (found heuristically to yield satisfactory results), we vary \[
t_{\text{close}} \in \{1.5, 2.0, 2.5, 3.0, 3.5, 4.0, 4.5, 5.0\}
\quad \text{and} \quad t_{\text{far}} \in \{0.30, 0.35, 0.40\}.
\]

\textbf{Results:} For each of the 100 data configurations with any choice of $t_{\text{close}}$ and $t_{\text{far}}$, the FMC algorithm achieved a mean clustering accuracy of $100\%$ and a mean normalized dispersion ratio of $1$; standard deviation was negligible in both cases.

\textbf{Setup 2:} We constructed 500 benchmark datasets with $k=2$ balls in $\mls{SPD}(4)$. For the ``close case'' we set $d_{\text{low}} = 1.10$ and $d_{\text{up}} = 1.15$. Similarly to the experiment above, we selected
\[
t_{\text{close}} \in \{1.5, 2.0, 2.5, 3.0, 3.5, 4.0, 4.5, 5.0\}.
\]
For the ``far case'' we set $d_{\text{low}} = 2.50$ and $d_{\text{up}} = 5.00$ and choose $t_{\text{far}} \in \{0.30, 0.35, 0.40\}$.

Since our algorithm relies on random sampling to estimate the Riemannian radii, we repeated the experiment for each choice of $t_{\text{far}}$ and $t_{\text{close}}$ 100 times. For each choice, we set the reference points to $R_1 = \bar{\gamma}_{12}(-t_{ij})$ and $R_2 = \bar{\gamma}_{12}(-t_{ij})$, respectively, as described above.

\textbf{Results:} Across all runs, we consistently observed that the clustering accuracy and the normalized dispersion are equal to $1$.

\textbf{Conclusion:} The FMC algorithm achieved a very high accuracy and normalized dispersion for all choices of $t_{\text{close}}$ and $t_{\text{far}}$ considered in this experiment.

\subsubsection{Comparison of reference point selection strategies}

\textbf{Purpose:} To compare the performance of the FMC algorithm under the two strategies for the selection of reference points presented in \Cref{s:ref_selection}.

\textbf{Setup:} We adopted the same strategy outlined in the above experiment to generate the benchmark data sets in $\mls{SPD}(4)$. We considered two strategies to select the reference points.
For the principled reference point selection strategy of~\Cref{s:principled_ref_selection}, we followed the same approach as in the above experiment and selected 20 reference points after fixing $t_{\text{close}}=2$ and $t_{\text{far}}=0.35$. For the simpler random selection strategy presented in \Cref{s:rand_ref_selection}, the 20 reference points were selected randomly from the dataset.

\textbf{Results:} The FMC algorithm in combination with the principled reference point selection strategy of~\Cref{s:principled_ref_selection} achieved a mean clustering accuracy of $100\%$ (with negligible standard deviation). In contrast, when we used the simpler random selection strategy of~\Cref{s:rand_ref_selection},  mean accuracy dropped to $98.48\%$ with a standard deviation of $0.012$. If the number of reference points increases to $30$, the mean accuracy is $98.95\%$ with a standard deviation of $0.007$.

\textbf{Conclusion:} The reference-point placement informed by the approximate Fr\'{e}chet means
of~\Cref{s:principled_ref_selection} is more sample-efficient and stable than naive random selection.

\subsection{A more challenging configuration}
\label{ssec:logEucl_vs_Frechet}

\textbf{Purpose:} To examine the impact of choosing $1-$ vs $2-$ Fr\'{e}chet map in the FCM approach and compare the FCM approach with the log-Euclidean (LEC) embedding.

\textbf{Setup:} We considered 4 disjoint balls of fixed radii $r=1$ in $\mls{SPD}(4)$, placed in a challenging configuration. Namely, we generated the center of the first ball according to $C_1 = \exp(V_1)$ where $V_1$ is a random matrix in $\mls{Symm}(4)$ (with each entry drawn from a centered normal distribution) which is then rescaled so that $\|V_1\|_F = 12$. In other words, $C_1$ is obtained by moving from $\Id$ in a random direction up to a distance of $12$ from $\Id$. We then selected the second center as $C_2 = \exp(V_2)$ where $V_2 = V_1 + \Delta$ is a small random perturbation of $V_1$ with $\Delta$ being a random matrix of $\mls{Symm}(4)$ with entries drawn from a centered normal distribution of variance $0.1$. To ensure that the resulting Riemannian balls $B(C_1,1)$ and $B(C_2,1)$ remain disjoint, we only kept those two centers when $d(C_1,C_2) >2$. The centers of the last two balls were then set to $C_3 = C_1^{-1}=\exp(-V_1)$ and $C_4=C_2^{-1}=\exp(-V_2)$. This ensures, on the one hand, that all 4 balls are disjoint from one another since $d(C_3,C_4) = d(C_1,C_2)>2$ (the equality itself follows from the expression of the distance \Cref{eq:SPD_distance}), while also making the configuration of those balls symmetric and thus centered around $\Id$.

We generated 250 realizations of the above scheme, and, for each ball, generated 5000 samples drawn uniformly inside it according to the same approach described in the above experiments. We then evaluated the clustering accuracy and normalized dispersion of the 1-Fr\'{e}chet and 2-Fr\'{e}chet map (1-FCM and 2-FCM) methods (using the reference point selection strategy described in \Cref{s:principled_ref_selection} and deployed in the numerical example above), as well as the LEC approach.

\textbf{Results:} Results are reported in \Cref{tab:results_4balls_LE_vs_Fr}.

\textbf{Conclusion:} The 1-FCM method achieved the best accuracy and dispersion, outperforming the 2-FCM and LEC approaches. We explain the poor performance of LEC (worst performance) by the fact that geodesic distances are not well approximated in the log-Euclidean setting for points far away from $\Id$, as is the case here. The squared Fr\'{e}chet map method leads to slightly improved accuracy and dispersion, yet still appears to suffer from the distortion induced by the mapping. This experiment
shows that the Fr\'{e}chet map method, especially the 1-Fr\'{e}chet map, produces much lower metric distortion (cf. \Cref{prop:Lipschitz_reg_Frechet}) compared to LEC. The difference becomes particularly evident when the manifold data configuration is not well reflected by its tangent space approximation.

\begin{table}
\caption{Comparison of 1-Fr\'{e}chet map (1-FCM), squared Fr\'{e}chet map (2-FCM), and the log-Euclidean clustering (LEC) methods. We report the mean accuracy (ground truth labels are available) and dispersion for the three methods over $250$ draws of four balls in $\mls{SPD}(4)$, selected randomly according to the scheme described in \Cref{ssec:logEucl_vs_Frechet}. We report the mean accuracy and the mean dispersion along with their standard deviations $\sigma$ (in brackets).}
\label{tab:results_4balls_LE_vs_Fr}
\tabadjust
\begin{tabular}{lrr}
\toprule
{\bf Method} & {\bf Mean Accuracy} (std $\sigma$) & {\bf Mean Dispersion} (std $\sigma$)\\
\midrule
1-FMC (ours) & 89.61\% ($17.27\%$) &  2.11 ($2.25$)\\
2-FMC (ours) & 69.52\% ($20.64\%$) &  5.58 ($4.39$)\\
LEC          & 62.08\% ($6.44\%$)  & 10.39 ($6.26$)\\
\bottomrule
\end{tabular}
\end{table}

\subsection{Efficiency comparison with IRC and ARC}

\textbf{Purpose:} To compare the proposed FCM algorithm against the IRC and ARC algorithms on high-dimensional SPD data.

\textbf{Setup:} We generated four disjoint balls in $\mls{SPD}(20)$ where the radius of each ball is randomly selected within the interval $[0.8, 1.2]$. We sampled 4\,000 points uniformly within each ball. For the implementation of the FCM algorithm, we generated the reference points using the strategy presented in~\Cref{s:principled_ref_selection} and already utilized in the above numerical experiments.

\textbf{Results:} \Cref{tab:spd20_results} reports the clustering performance of the FCM, IRC, and ARC algorithms. Performance numbers, including average accuracy, normalized dispersion, and running time.
are averages over $50$ replications.

\textbf{Conclusion:} The FCM algorithm with $t_{\text{close}}=5$ and $t_{\text{far}}=0.35$ achieves clustering accuracy comparable to IRC and ARC, while being approximately fourteen times faster than IRC and nine times faster than ARC.

\begin{table}
\caption{Comparison of FCM, ARC, and IRC on synthetic data sampled from four disjoint balls in $\mls{SPD}(20)$. We report the runtime (in seconds), the accuracy, and the normalized dispersion (from left to right).}
\label{tab:spd20_results}
\tabadjust
\begin{tabular}{lrrrr}
\toprule
\bf Method & \bf Runtime & \bf Speedup & \bf Accuracy & \bf Normalized Dispersion\\
\midrule
IRC          & 832 & ---         &  94.1\% & 1.17\\
ARC          & 513 & 8.8$\times$ &  91.6\% & 1.26\\
2-FCM (ours) &  58 & 14.3$\times$  & 100.0\% & 1.00\\
\bottomrule
\end{tabular}
\end{table}

\subsection{FCM algorithm on the texture dataset}

\textbf{Purpose:} To test the performance of the FCM algorithm on real data. Here, we consider images of texture.

\begin{figure}
\centering
\includegraphics[width=0.2\textwidth]{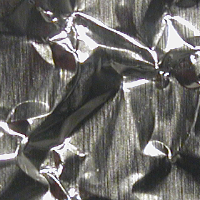}
\includegraphics[width=0.2\textwidth]{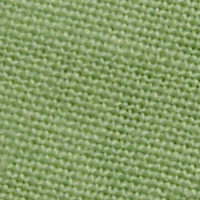}
\includegraphics[width=0.2\textwidth]{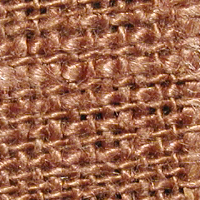}
\includegraphics[width=0.2\textwidth]{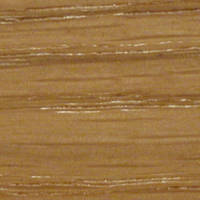}
\caption{Representative datasets from the four textures (aluminum foil, cotton, linen, and wood) considered in this study. These images are taken from the KTH-TIPS2b dataset.}\label{fig:kth-tips2b}
\end{figure}

\textbf{Setup:} We select four categories (see \Cref{fig:kth-tips2b}) from the KTH-TIPS2b texture dataset~\cite{caputo2005class}. Each category has $432$ images. We resize the original images to $128\times 128$ pixels. We generate covariance descriptors as follows. For each pixel location $(u,v)$ in the $128 \times 128$ image, where $1 \le u,v \le 128$, we compute a 23-dimensional feature vector
\[
x_{u,v}^m=\left[r_{u,v}, g_{u,v}, b_{u,v}, \left|G_{u,v}^{0,0}(m)\right|,\cdots, \left|G_{u,v}^{4,3}(m)\right|\right]^\mathsf{T} \in \mathbb{R}^{23},
\]
 where $G_{u,v}^{o,s} \in \mathbb{R}$, $0 \le o \le 4$, $0 \le s \le 3$ are the Gabor filter coefficients  of the image $m \in \{1,\ldots,1728\}$ centered at $(u,v)$ defined in Section 2.1 of \cite{tou2008gabor}. Next, for each image, we compute the $23\times 23$ covariance matrix
\[
C_m = \frac{1}{N-1} \sum_{u=1}^{128} \sum_{v=1}^{128} (x_{u,v}^m-\mu_{u,v}^m)(x_{u,v}^m-\mu_{u,v}^m)^\mathsf{T} \in \mathbb{R}^{23,23},
\]
 where $\mu_{u,v}$ is the mean of feature vector $x_{u,v}\in\mathbb{R}^{23}$. In summary, we obtain a set of $1\,728$ covariance matrices $C_k$, $k=1,\ldots,1\,728$ of size $23 \times 23$.

We applied four clustering methods: the proposed FCM algorithm with $t_\text{close}=5$ and $t_\text{far}=0.35$, and the algorithms ARC, IRC, and LEC.

\textbf{Results:} \Cref{tab:texture_results} reports the results (averaged over $50$ repetitions to get more precise estimates), including the average runtime, accuracy, and normalized dispersion.

\textbf{Conclusion:} The FCM algorithm achieves results comparable to those of the ARC and IRC methods in terms of clustering performance. The FMC approach is significantly faster than ARC and IRC, with LEC being slightly faster than FCM, but not as accurate. LEC is faster than FCM but slightly less accurate.

\begin{table}
\caption{Comparison of FCM, ARC, and IRC on the KTH-TIPS2b texture dataset. We report the runtime (in seconds), the speedup, the clustering accuracy (ground truth textures are available), and the normalized dispersion.}
\tabadjust
\begin{tabular}{lrrrr}
\toprule
\bf Method & \bf Runtime & \bf Speedup & \bf Accuracy & \bf Normalized Dispersion\\
\hline
IRC    &   174 & ---          &     77.42\% & 0.68 \\
ARC    &    66 &  2.6$\times$ & \bf 77.67\% & 0.69 \\
1-FCM  &     7 & 24.9$\times$ &     76.33\% & 0.69 \\
2-FCM  &     8 & 21.8$\times$ &     76.15\% & 0.69 \\
LEC    & \bf 3 & 58.0$\times$ &     75.84\% & 0.69 \\
\hline
\end{tabular}
\label{tab:texture_results}
\end{table}

\subsection{FCM algorithm on the COBRE resting-state fMRI dataset}

\textbf{Purpose:} To test the performance of the FCM algorithm on real data. Here, we consider an fMRI dataset.

\textbf{Setup:} We used the COBRE resting-state fMRI dataset from~\cite{cobre_ml_figshare}, which contains two diagnostic categories, schizophrenia (SCZ) and healthy control (Control), for a total of $146$ subjects in total (72 SCZ and 74 Control). Using the released precomputed functional connectivity features with 197-voxel region-of-interest resolution, we calculated one $197\times 197$ SPD matrix per subject.

Each subject's connectivity matrix is provided as an upper-triangular matrix without the diagonal. We recovered the full symmetric matrix $A$ by completion, setting the diagonal entries to $1$ (correlation convention), and adding a small perturbation $\epsilon >0$ to the identity to ensure strict positive-definiteness. That is, we obtained SPD matrices as
\[
A \leftarrow \tfrac{1}{2}(A + A^\mathsf{T}) + \epsilon\Id
\]
with $\epsilon = \snum{1e-6}$.

We applied four clustering approaches on the COBRE dataset: our proposed FCM algorithm with $t_{\text{close}}=5$ and $t_{\text{far}}=0.35$, and the algorithms ARC, IRC, and LEC.

\textbf{Results:} \Cref{tab:cobre_results} reports the mean performance over $10$ repeated runs, including average accuracy, normalized dispersion, runtime, and speedup factor.

\textbf{Conclusion:} All algorithms achieve comparable accuracy, while FCM and LEC exhibit significantly lower computational cost. We note that the accuracy for all approaches is fairly low in this particular example, while the normalized dispersions are all below $1$, which reflects the fact that the two groups in this dataset are not necessarily well-separated, at least based on the k-means approach.

\begin{table}
\caption{Comparison of FCM, ARC, LEC, and IRC on the COBRE dataset. We report (from left to right) the runtime (in seconds), the speedup compared to IRC, the accuracy, and the normalized dispersion.}
\label{tab:cobre_results}
\tabadjust
\begin{tabular}{lrrrr}\toprule
\bf Method & \bf Runtime & \bf Speedup & \bf Accuracy & \bf Normalized Dispersion\\
\midrule
IRC          & 11\,976 & ---         & 60.96\%     & 0.95 \\
ARC          &     203 &  $60\times$ & 61.73\%     & 0.88 \\
1-FCM (ours) &      20 & $600\times$ & 61.05\%     & 0.88 \\
2-FCM (ours) &      20 & $600\times$ & \bf 61.73\% & 0.88 \\
LEC          &  \bf 18 & $600\times$ & 61.25\%     & 0.88 \\
\bottomrule
\end{tabular}
\end{table}

\section{Conclusion and future perspectives}\label{s:conclusions}

We introduced and evaluated a new variant of the classical k-means method for clustering data on non-Euclidean spaces. Our approach is simple to implement, easy to parallelize, and fast, while able to closely match the results of the intrinsic k-means scheme on the manifold with some adequate choice of reference points. The main observations are:
\begin{itemize}
\item The runtime of the proposed algorithms is at least one order of magnitude faster than performing clustering on the manifold without sacrificing accuracy.
\item The performance of the proposed approach is consistent across a series of synthetic and real-data examples. Although the Log-Euclidean embedding approach remains slightly faster, its accuracy can deteriorate for certain data configurations. We demonstrated empirically that the accuracy of our method remains more stable in such situations.
\item We investigated some theoretical properties of Fr\'{e}chet maps that could provide a fundamental mathematical underpinning to our proposed FMC clustering approach. Although one can derive a quite complete picture for Fr\'{e}chet maps on Euclidean spaces, we found that a similar analysis becomes much more challenging on the $\mls{SPD}(n)$ spaces and a fortiori on more general Cartan-Hadamard manifolds. However, we anticipate that this preliminary exploration will draw the attention of the applied mathematics community and pave the way for future work on this topic.
\end{itemize}
From a wider perspective, we believe that Fr\'{e}chet maps could be further leveraged as a general parametric class of mappings on manifolds for applications beyond the clustering problem considered in this work, and be used within e.g. autoencoder architectures to estimate optimal latent feature space embeddings for manifold datasets.

\section*{Acknowledgments}
The authors thank David Pierucci for his help in the characterization of totally geodesic hypersurfaces of $\mls{SPD}(n)$, as well as Alice Barbora Tumpach and Xavier Pennec for several stimulating discussions on topics related to this work.

\begin{appendix}

\section{Hardware and Software Libraries}\label{s:compute}

All experiments were executed on the Carya Cluster---a modern computing system hosted by the Research Computing Data Core of the Hewlett-Packard Enterprise Data Science Institute of the University of Houston. Carya hosts a total of \inum{9984} Intel CPU cores and \inum{327680} Nvidia GPU cores integrated within 188 compute and 20 GPU nodes, equipped with Intel Xeon G6252 CPUs and NVIDIA V100 GPUs.  All nodes are equipped with solid-state drives for local high-performance data storage.
Our code was implemented in \texttt{Python} version 3.12.3. Some of the modules used in our code are based on \texttt{scikit-learn} version {1.5.1} and \texttt{scikit-image} version {0.24.0}. Upon acceptance of this article, our code will be released on \texttt{GitHub} at \url{https://github.com/jishi24/Frechet-Clustering}.

\section{Proof of Theorem \ref{thm:convex_image_Frechet}}
\label{app:proof_image_balls}
Let us first examine the characterization of the image of the Fr\'{e}chet map $F_r$. Let $d =(d_1, \dots, d_m) \in \R^m$ such that $d = F_r(x)$ for some $ x \in \R^m$, i.e., $d_i = \|x-r_i\|^2$, $i=1,\ldots,m$. By translation invariance, we may assume that $r_m = 0$. Then, due to the assumption made on the reference points, the vectors $r_i=r_i-r_m$ for $i=1,\ldots,m-1$ are linearly independent. Their Gram matrix $G = (r_i^\mathsf{T} r_j) \in  \R^{(m-1)\times (m-1)}$ is thus positive definite. We then define the vectors $z, b$, and $u$ in $\R^{m-1}$ by $z_i = d_i - d_m = \|r_i\|^2-2 r_i^\mathsf{T}x$, $b_i = \|r_i\|^2$, and $u_i=r_i^\mathsf{T} x$ for all $i=1,\ldots,m-1$. It results that $u = \frac{1}{2}(b-z)$.

We now aim to express $x$ with respect to $u$ and $d_m$. To that end, we write $x = x_{H^0} + sn$, where $s \in \R$, $H^0$ is the hyperplane spanned by the $r_i$'s, $n$ the unit normal vector to that hyperplane, and $x_{H^0}$ is the projection of $x$ onto $H^0$. Since $u$ is the vector of inner products of $x$ and the $r_i$'s for $i=1,\ldots,m-1$, one immediately has that $x_{H^0} = \sum_{i=1}^{m-1} \alpha_i r_i$ with $\alpha = G^{-1} u$, from which we also get $\|x_{H^0}\|^2 = u^\mathsf{T} G^{-1} u$. Then $\|x\|^2 = \|x_{H^0}\|^2 + s^2$, i.e., $s^2 = d_m - u^\mathsf{T} G^{-1} u$. It follows that a vector $d$ in $\R^m$ belongs to the image $F_r(\R^m)$ if and only if $u^\mathsf{T} G^{-1} u \leq d_m$; that is:
\begin{equation*}
    F_r(\R^m) = \left\{d\in \R^m: \ \begin{pmatrix}\|r_1\|^2 + d_m -d_1 \\ \vdots \\ \|r_{m-1}\|^2 + d_{m} - d_{m-1} \end{pmatrix}^\mathsf{T} G \begin{pmatrix}\|r_1\|^2 + d_m -d_1 \\ \vdots \\ \|r_{m-1}\|^2 + d_{m} - d_{m-1} \end{pmatrix} \leq 4 d_m \right\}.
\end{equation*}
We note that the above expression is the equation of the interior of a paraboloid in $\R^m$, proving the first claim in Theorem \ref{thm:convex_image_Frechet}. The above equations also provide the inverse of $F_r$ for $d$ in the above set, which is obtained by taking $s = \pm \sqrt{d_m - u^\mathsf{T} G^{-1} u}$, with the two possible solutions corresponding to two symmetric points $x^-,x^+$ on each side of the hyperplane $H^0$ as illustrated in \Cref{Frechet_Euc.fig}.

Assume now that $B(x_0,\rho_0)$ is a ball such that $B(x_0,\rho_0) \subset H^+$ (the same argument applies in $H^-$). Denoting $x_0^\bot = n^\mathsf{T}x_0$ the component of $x_0$ normal to $H^0$ which is also the distance from $x_0$ the the hyperplane, we have $x_0^\bot > \rho_0$. Using the notations introduced previously, the condition that $x=F_r^{-1}(d)$ belongs to $B(x_0,\rho_0)$ is equivalent to:
\begin{equation*}
\begin{aligned}
   \rho_0^2 \geq \|x-x_0\|^2 &= d_m -2 x^\mathsf{T} x_0 + \|x_0\|^2 \\
   &=d_m - 2 x_{H^0}^\mathsf{T} x_0 -2s n^\mathsf{T}x_0 + \|x_0\|^2 \\
   &=d_m - 2\sum_{i=1}^{m-1} (G^{-1} u)_i r_i^\mathsf{T} x_0 -2 \sqrt{d_m - u^\mathsf{T}G^{-1} u} \, x_0^{\bot}.
   \end{aligned}
\end{equation*}

Let us define the vector $u_0 \defeq (r_1^\mathsf{T} x_0, \ldots,r_{m-1}^\mathsf{T}x_0)^\mathsf{T} \in \R^{m-1}$. Then the second term on the right-hand side can be rewritten according to:
\begin{equation*}
    \sum_{i=1}^{m-1} (G^{-1} u)_i r_i^\mathsf{T} x_0 = u_0^\mathsf{T} G^{-1} u.
\end{equation*}

\noindent Therefore, $x \in B(x_0,\rho_0)$ if and only if
\begin{equation*}
    \sqrt{d_m - u^\mathsf{T}G^{-1} u} \, x_0^{\bot} \geq \frac{d_m}{2} -u_0^\mathsf{T} G^{-1} u +\frac{\|x_0\|^2-\rho_0^2}{2},
\end{equation*}

\noindent which, by squaring both sides, is also equivalent to
\begin{equation*}
    (x_0^{\bot})^2(d_m -u^\mathsf{T} G^{-1} u) \geq \left(\frac{d_m}{2} - u_0^\mathsf{T} G^{-1} u + \frac{\|x_0\|^2-\rho_0^2}{2} \right)^2.
\end{equation*}

This in turn can be expressed as
\begin{equation}
\label{eq:app_proof_2.6}
   (x_0^{\bot})^2 u^\mathsf{T} G^{-1} u + \left(\frac{d_m}{2} - u_0^\mathsf{T} G^{-1} u \right)^2 + g(d) \leq \text{C},
\end{equation}

\noindent where $g:\R^m \rightarrow \R$ is some linear function of $d$ and $\text{C} \in \R_+$ is a constant that we do not explicitly specify to keep the expression compact. As $(u,d_m)$ is itself a linear invertible function of $d$, the above describes the interior of a quadric of $\R^m$. Furthermore, since $x_0^\bot >0$ and $G^{-1}$ is positive definite, we see that $(u,d_m) \mapsto (x_0^{\bot})^2 u^\mathsf{T} G^{-1} u + \left(\frac{d_m}{2} - u_0^\mathsf{T} G^{-1} u \right)^2$ is a positive definite quadratic form on $\R^m$. Thus, the set of $d$ that satisfies \Cref{eq:app_proof_2.6} is the interior of an ellipsoid. We conclude that $F_r(B)$ is the interior of an ellipsoid of $\R^m$ (intersected with the interior of the paraboloid from the previous paragraph).

Finally, the last statement in the theorem is simply a consequence of the Hahn--Banach theorem since $F_r(B_1)$ and $F_r(B_2)$ are two convex subsets of $\R^m$ (as the intersections of the inside of a paraboloid and an ellipsoid) and are also disjoint owing to the injectivity of $F_r$ on $H^+$ given by \Cref{thm:invertibility_Frechet_Eucl}.

\section{Proof of \Cref{thm:convex_image_Frechet1}}
\label{app:proof_convex_image_Frechet1}
Unlike the case of 2-Fr\'{e}chet maps, it is not as simple to geometrically characterize the image of a ball under $F_r^1$. Instead, our proof of convexity relies on showing the positivity of the curvature of the image set boundary. Let us assume, without loss of generality, that $A \subset H^+$.  First of all, from Remark \ref{rem:invertibility_Frechet_Eucl}, we know that $F_r^1$ is a diffeomorphism on the halfspace $H^+$. Let $B$ be a ball in $A$, of radius $\rho>0$, and $S = \partial B$ the boundary sphere. We will use the following expression for the second fundamental form of the image $F_r^1(S)\subset \R_{+}^m$, which follows from standard results in Riemannian geometry of embedded submanifolds:
\begin{lemma}
For any $x \in S$, and $u,v \in T_x S$, the second fundamental form of the hypersurface $F_r^1(S)$ satisfies
\begin{equation}
\label{eq:2nd_fundamental_form}
    \mathrm{I\!I}_{F_r^1(x)}(DF_r^1(x) \cdot u,DF_r^1(x) \cdot v) = \langle D^2F_r^1(x)(u,v) + \frac{1}{\rho} \langle u , v \rangle DF_r^1(x) \cdot n, n'\rangle,
\end{equation}
where $n$ denotes the unit outward normal vector to $S$ at $x$ and $n'$ the unit outward normal to $F(S)$ at $F(x)$.
\end{lemma}
Now, given two linearly independent tangent vectors $u,v$ to $S$ at $x$ and their pushforward $u' = DF_r^1(x) \cdot u, v' = DF_r^1(x) \cdot v \in T_{F_r^1(x)}F_r^1(S)$, the sectional curvature of $F_r^1(S)$ associated to the plane spanned by $(u',v')$ is classically given by (c.f. \cite{do1992riemannian}):
\begin{equation}
\label{eq:expr_sec_curvatures}
    K(u', v') = \frac{\mathrm{I\!I}(u',u') \mathrm{I\!I}(v',v') - \mathrm{I\!I}(u',v')^2}{\|u'\|^2\|v'\|^2 - \langle u' ,v' \rangle^2}
\end{equation}
where we dropped the dependency in $x$ and $F_r^1(x)$ to lighten notation. Since the denominator is positive for any linearly independent vectors $u,v$, we just need to show that, under the adequate conditions on the reference points, the second fundamental form remains positive definite. To do so, we need the following expressions for the first and second-order differentials of the Fr\'{e}chet map:
\begin{equation*}
    DF_r^1(x) \cdot u = \begin{pmatrix} \langle \frac{x-r_1}{\|x-r_1\|} , u \rangle \\ \vdots \\ \langle \frac{x-r_m}{\|x-r_m\|} , u \rangle \end{pmatrix}, \ \ D^2F_r^1(x)(u,v) = \begin{pmatrix} \frac{1}{\|x-r_1\|} \langle \Pi_{1} u, v \rangle \\ \vdots \\ \frac{1}{\|x-r_m\|} \langle \Pi_{m} u, v \rangle \end{pmatrix}
\end{equation*}
where $\Pi_i$, for $i=1\ldots,m$, denotes the orthogonal projection onto the hyperplane perpendicular to $x-r_i$. Then, using \eqref{eq:2nd_fundamental_form} and the fact that $n'=\frac{DF_r^1(x)^{-T}(n)}{\|DF_r^1(x)^{-T}(n)\|}$, we see that for all $x \in S \subset A$ and tangent vector $u \in T_x S$:
\begin{equation}
\label{eq:2nd_form_bound}
\begin{aligned}
    \mathrm{I\!I}(u',u') &= \frac{1}{\rho} \|u\|^2 \langle DF_r^1(x) \cdot n, n'\rangle + \langle D^2F_r^1(x)(u,u) , n' \rangle \\
    &=\frac{1}{\rho} \|u\|^2 (n^T DF_r^1(x)^{-1}DF_r^1(x)^{-T} n) + \sum_{i=1}^{m} \frac{n'_i}{\|x-r_i\|}  u^T \Pi_i u \\
    &\geq \frac{\sigma_{\text{min}}^2(DF_r^1(x))}{\rho} \|u\|^2 - \frac{\sqrt{m}}{d(r,A)} \|u\|^2
\end{aligned}
\end{equation}
where $d(r,A)= \min_{i=1,\ldots,m} \, d(r_i,A)$ denotes the distance of the reference point set to the compact $A$. We note that the right hand side term can be made arbitrarily small by selecting reference points far enough from $A$.  On the other hand, one still needs to control the decrease of the first term and specifically of $\sigma_{\text{min}}^2(DF_r^1(x))$, the smallest singular value of the first differential of $F_r^1$. Given the expression of $DF_r^1$ computed above, this corresponds to the smallest eigenvalue of the Gram matrix for the $m$ unit vectors $\left(\frac{x-r_1}{\|x-r_1\|},\ldots \frac{x-r_m}{\|x-r_m\|} \right)$. By the Gershgorin circle theorem, this eigenvalue is bounded from below by
\begin{equation*}
\begin{aligned}
   \sigma_{\text{min}}^2(DF_r^1(x)) &\geq 1 - \max_{i=1,\ldots,m} \sum_{j \neq i} \left|\left\langle \frac{x-r_i}{\|x-r_i\|} , \frac{x-r_j}{\|x-r_j\|} \right \rangle \right| \\
   &\geq 1 - \max_{i=1,\ldots,m} (m-1) \max_{j, j\neq i} \left|\left\langle \frac{x-r_i}{\|x-r_i\|} , \frac{x-r_j}{\|x-r_j\|} \right \rangle \right| \\
   &\geq 1 - (m-1) \max_{j \neq i} \left|\left\langle \frac{x-r_i}{\|x-r_i\|} , \frac{x-r_j}{\|x-r_j\|} \right \rangle \right|.
\end{aligned}
\end{equation*}
Now, by taking the minimum of these lower bounds over the compact set $A$, we get that for all $x \in A$, $\sigma_{\text{min}}^2(DF_r^1(x)) \geq 1-(m-1)\mu$. Thus, going back to \eqref{eq:2nd_form_bound}:
\begin{equation*}
    \mathrm{I\!I}(u',u') \geq \left(\frac{1-(m-1)\mu}{\rho} -\frac{\sqrt{m}}{d(r,A)} \right) \|u\|^2.
\end{equation*}
Therefore, the second fundamental form is positive definite as soon as $\frac{\rho}{d(r,A)} < \frac{1-(m-1)\mu}{\sqrt{m}}$. In that case, all sectional curvatures $K(u',v')$ in \eqref{eq:expr_sec_curvatures} are positive. Since $F_r^1(S)$ is the diffeomorphic image of a sphere, it is a complete manifold and, thus, by the result of \cite{sacksteder1960hypersurfaces}, we deduce that $F_r^1(B)$ is convex.

\section{Proof of \Cref{prop:tot_geod_SPD}}
\label{app:proof_tot_geod}
Here we derive the complete list of the totally geodesic hypersurfaces of $\mls{SPD}(n)$ for $n\geq 3$ as stated in the proposition, and also cover the case $n=2$ for completeness. Let us first state an equivalent to the Lie triple system condition characterizing totally geodesic submanifolds in $\mls{SPD}(n)$:

\begin{theorem}[Corollary 1.2 in \cite{tumpach2024totally}]
Any totally geodesic submanifold of $\mls{SPD}(n)$ is of the form $S = Q\exp(H)Q$ with $Q \in \mls{GL}(n)$ and $H$ a subspace of $\mls{Symm}(n)$ that satisfies:
\begin{equation*}
    [V,[V,W]] \in H, \ \ \text{for any }  V,W \in H,
\end{equation*}

\noindent where $[V,W]\doteq VW - WV$ is the usual Lie Bracket on matrices.
\label{thm:Lie_triple_sys}
\end{theorem}

We further point out that the above Lie triple system condition is also equivalent (via a simple linearity argument) to the fact that $[[U,V],W] \in H$ for any $U,V,W \in H$. Based on this theorem, we can thus first focus on determining which codimension $1$ subspaces $H$ satisfy the above Lie triple system condition.
\vskip1ex
$\bullet$ \textbf{Characterization of $H$:} let $H$ be a $(m-1)$-dimensional subspace of $\mls{Symm}(n)$ such that $[V,[V,W]] \in H$ for any $V,W \in H$. There exists $A \in \mls{Symm}(n)$ with $A \neq 0$ for which $H = \{V \in \mls{Symm}(n): \ \Tr(A V) = 0\}$, i.e., $H$ is the orthogonal subspace to $A$.  We may diagonalize $A$ as $A=R^\mathsf{T}DR$ with $D=\operatorname{diag}(d_i)_{i=1,\ldots,n}$ and $R \in O(n)$. Then $\Tr(AV) = 0$ is equivalent to $\Tr(DRVR^\mathsf{T}) = 0$ so we can write $H=R^\mathsf{T} H_D R$, where $H_D$ is the space of symmetric matrices orthogonal to $D$. It is then clear that $H$ is totally geodesic if and only if $H_D$ is totally geodesic since $P\mapsto R^\mathsf{T} P R$ is an isometry of $\mls{SPD}(n)$.

We first treat the case $n\geq 3$. Since $D \neq 0$, without loss of generality, we assume that $d_1 \neq 0$. Let us denote by $\{E_{ij}\}_{i,j=1,\ldots,m}$ the canonical basis of $\R^{m \times m}$, i.e., $E_{ij} = (\delta_{\{k=i,l=j\}})_{k,l}$. Then one can check that the following is a basis for $H_D = \{V \in \mls{Symm}(n): \ \Tr(DV) = 0 \}$:
\begin{equation*}
    \mathcal{B} = \left\{\left(E_{ij}+E_{ji}\right)_{i\neq j},\left(\frac{d_i}{d_1}E_{11} - E_{ii}\right)_{i\geq 2} \right\}.
\end{equation*}

\noindent In particular, if $H_D$ is totally geodesic then for any $i,j$ with $i\neq 1$, $j\neq 1$ and $i\neq j$ (which exist since $n \geq 3$), we must have:
\begin{equation}
\label{eq:trip_Lie_brack_n}
\left[E_{1,j}+E_{j,1},\left[\frac{d_i}{d_1}E_{1,1} - E_{i,i},E_{1,j}+E_{j,1}\right]  \right] \in H_D.
\end{equation}

\noindent After calculations, we find that:
\begin{equation*}
\begin{aligned}
    \left[\frac{d_i}{d_1}E_{1,1} - E_{i,i},E_{1,j}+E_{j,1}\right] &= \frac{d_i}{d_1}\left(E_{1,j} - E_{j,1}\right) \\
    \left[E_{1,j}+E_{j,1},\left[\frac{d_i}{d_1}E_{1,1} - E_{i,i},E_{1,j}+E_{j,1}\right]  \right] & = \frac{d_i}{d_1}\left[E_{1,j}+E_{j,1},E_{1,j} - E_{j,1}\right] \\
                & = 2\frac{d_i}{d_1}\left(E_{j,j}-E_{1,1}\right).
\end{aligned}
\end{equation*}

\noindent from which it follows that condition \Cref{eq:trip_Lie_brack_n} is equivalent to $2\frac{d_i d_j}{d_1} - 2\frac{d_i d_1}{d_1} =0$, i.e., $d_j=d_1$ for any $j \neq 1$. Thus, the matrix $D$ is of the form $\lambda \Id$ for some $\lambda \neq 0$ and so is $A = R^\mathsf{T} D R$. It follows that $H$ is the space of matrices with zero trace.

In the case $n=2$, without loss of generality, one may write, up to a scaling factor, $D = \begin{pmatrix}1 & 0 \\ 0 & d \end{pmatrix}$, where $d \in \R$ remains to be determined. We see that a basis for the linear subspace $H_D$ is then given by:
\begin{equation*}
\mathcal{B} = \left\{B_1=\begin{pmatrix}0 & 1 \\ 1 & 0 \end{pmatrix}, B_2=\begin{pmatrix}d & 0 \\ 0 & -1 \end{pmatrix} \right\}.
\end{equation*}

\noindent By the remark below, \Cref{thm:Lie_triple_sys}, and the bilinearity of the Lie bracket operation, the triple Lie bracket condition is also equivalent to $U,V,W \in \mathcal{B} \implies [U,[V,W]] \in H_D$. Computation of the different possible triple Lie brackets leads to:
\begin{equation*}
\begin{aligned}
    &[B_1,B_2] = B_1 B_2 - B_2 B_1 = \begin{pmatrix}
        0 & -d-1\\ d+1 & 0
    \end{pmatrix},\\
    &[B_1,[B_1,B_2]] = B_1 [B_1,B_2] - [B_1,B_2] B_1 =\begin{pmatrix}
        2d +2 & 0\\ 0 & -2d-2
    \end{pmatrix},\\
    &[B_2,[B_1,B_2]] = B_2 [B_1,B_2] - [B_1,B_2] B_2 = -\begin{pmatrix}
        0 & (d+1)^2\\ (d+1)^2 & 0
    \end{pmatrix}.
\end{aligned}
\end{equation*}

\noindent In particular, if $H_D$ is totally geodesic then $[B_1,[B_1,B_2]]\in H_D$, i.e., $\Tr(D[B_1,[B_1,B_2]])=0$, which gives the condition $d=-1$ or $d=1$. Conversely, for $d=-1$ or $d=1$, the triple Lie bracket condition is clearly satisfied, and we obtain a totally geodesic space. This in turn means that either $D = \Id$ or $D=\begin{pmatrix}1 & 0 \\ 0 & -1 \end{pmatrix}$, which is the matrix of the reflection by the $x$ axis. In the first case, $H$ is again the space of matrices of $\mls{Symm}(n)$ with zero trace. In the second case, we see that $H_D = \left\{a \Id + b J: \ a,b \in \R \right\}$ where $J=\begin{pmatrix}0 & 1 \\ 1 & 0 \end{pmatrix}$, from which one gets that $H = \left\{R^\mathsf{T} \begin{pmatrix}a & b \\ b & a \end{pmatrix} R: \ a,b \in \R\right\}$.
\vskip1ex
$\bullet$ \textbf{Classification of totally geodesic hypersurfaces:} when $H$ is the subspace of zero trace matrices, $\exp(H)$ is the precisely submanifold $\mls{SPD}_1(n)$ of SPD matrices of determinant $1$. Then for any $Q \in \mls{GL}(n)$, one can easily check that $Q\exp(H)Q=\mls{SPD}_r(n)$ where $r = \det(Q)^2$. Based on the above characterization of $H$ and Theorem \ref{thm:Lie_triple_sys}, we deduce that for $n\geq 3$, the totally geodesic hypersurfaces of $\mls{SPD}(n)$ are exactly the $\mls{SPD}_r(n)$ for $r>0$.

For $n=2$, in addition to the $\mls{SPD}_r(2)$, we also have the totally geodesic submanifolds $Q\exp(H)Q$ for $Q \in \mls{GL}(2)$ and the subspaces $H = \left\{R^\mathsf{T} \begin{pmatrix}a & b \\ b & a \end{pmatrix} R: \ a,b \in \R\right\}$ with $R \in O(2)$. Furthermore, for any $a,b \in \R$,
\begin{equation*}
    \exp(R^\mathsf{T}(a \Id + b J)R) = R^\mathsf{T}\exp(a\Id + bJ) R = e^a R^\mathsf{T} \begin{pmatrix}\cosh(b) & \sinh(b) \\ \sinh(b) & \cosh(b) \end{pmatrix} R.
\end{equation*}

\noindent Moreover, we can see that the set of all matrices $e^a \begin{pmatrix}\cosh(b) & \sinh(b) \\ \sinh(b) & \cosh(b) \end{pmatrix}$ for $a,b \in \R$ is equal to the following cone in $\mls{SPD}(2)$:
\begin{equation*}
\mathcal{C} = \left\{\begin{pmatrix}\alpha & \beta \\ \beta & \alpha \end{pmatrix}: \ \alpha>|\beta|\right\}.
\end{equation*}

We deduce that the second family of totally geodesic submanifolds of dimension $2$ in $\mls{SPD}(2)$ are the $Q^\mathsf{T}\mathcal{C}Q$ for $Q \in \mls{GL}(2)$, i.e., the translated versions of the cone $\mathcal{C}$. Note that with the particular choice $Q = \frac{1}{\sqrt{2}} \begin{pmatrix}1 & 1 \\ -1 & 1 \end{pmatrix}$, one has $Q^T\mathcal{C}Q = \left\{ \begin{pmatrix}\alpha + \beta & 0 \\ 0 & \alpha-\beta \end{pmatrix}: \alpha > |\beta| \right\}$ which is exactly the set of all the diagonal matrices of $\mls{SPD}(2)$ and is an obvious example of a totally geodesic submanifold of dimension $2$.
\vskip1ex
$\bullet$ \textbf{Projection onto $\mls{SPD}_r(n)$:} lastly, we prove the expression of the projection $\Pi_r(P)$ of $P \in \mls{SPD}(n)$ onto the totally geodesic hypersurface $\mls{SPD}_r(n)$. From Theorem 4.2 in \cite{tumpach2024totally}, the projection map exists and is continuous on $\mls{SPD}(n)$. First, we see that $\Pi_r(P) \doteq \sqrt[n]{\frac{r}{\det(P)}} P \in \mls{SPD}_r(n)$, where we denote $\lambda(P) = \sqrt[n]{\frac{r}{\det(P)}}>0$. We note that the tangent space to the submanifold $\mls{SPD}_r(n)$ at $\Pi_r(P)$ is the subspace of symmetric matrices $V$ satisfying $\Tr(\Pi_r(P)^{-1}V) = 0$ which is the same as the set of $V \in \mls{Symm}(n)$ such that $\Tr(P^{-1}V) = 0$. Consequently, the orthogonal vectors to that subspace, with respect to the metric $\langle\,\cdot\,,\,\cdot\,\rangle_{\Pi_r(P)}$, are the $s\,P$ for $s \in \mathbb{R}$ since for any $V \in T_{\Pi_r(P)} \mls{SPD}_r(n)$:
\begin{equation*}
\begin{aligned}
    \langle sP , V \rangle_{\Pi_r(P)} &= \Tr(\Pi_r(P)^{-1}(sP)\Pi_r(P)^{-1} V) \\
    &= \frac{s}{\lambda(P)^2} \Tr(P^{-1}PP^{-1}V) = \frac{s}{\lambda(P)^2} \Tr(P^{-1}V) = 0.
\end{aligned}
\end{equation*}

\noindent Now, we have in addition that:
\begin{equation*}
\begin{aligned}
    \log_{\Pi_r(P)} P &= \Pi_r(P)^{1/2} \log(\Pi_r(P)^{-1/2} P \Pi_r(P)^{-1/2}) \Pi_r(P)^{1/2} \\
    &=\lambda(P) P^{1/2} \log(\lambda(P)^{-1} P^{-1/2} P P^{-1/2}) P^{1/2} \\
    &=\lambda(P) P^{1/2} \log(\lambda(P)^{-1} \Id) P^{1/2} \\
    &=-(\lambda(P) \log\, \lambda(P)) P.
\end{aligned}
\end{equation*}

\noindent frow which we deduce that $\log_{\Pi_r(P)} P$ is orthogonal to $T_{\Pi_r(P)} \mls{SPD}_r(n)$. If $Q$ is any point in $\mls{SPD}_r(n)$, the geodesic triangle $\Delta\, Q\Pi_r(P)P$ thus has a right angle at $\Pi_r(P)$. Since $\mls{SPD}(n)$ is a non-positively curved space, we have the classical inequality:
\begin{equation*}
    d(P,Q)^2 \geq d(P,\Pi_r(P))^2 + d(\Pi_r(P),Q)^2.
\end{equation*}

Therefore $\Pi_r(P) = \operatorname{arg\, min}_{Q \in \mls{SPD}_r(n)} \, d(P,Q)^2$ showing that $\Pi_r(P)$ is the projection of $P$ onto $\mls{SPD}_r(n)$.

\section{Proof of \Cref{thm:EBS_tot_geod}}
\label{app:proof_thm_EBS_tot_geod}
Without loss of generality, we can reduce the proof to the case $r=1$ and assume that $R_1,\ldots,R_m$ are all SPD matrices of unit determinant.
\vskip1ex
$\bullet$ $\mls{SPD}_1(n) \subseteq S_{R_1,\ldots ,R_m}$: for any $P \in \mls{SPD}_1(n)$, the log vectors $\log_P R_1,\ldots,\log_P R_m$ must all belong to the tangent space at $P$ of $\mls{SPD}_1(n)$ since the latter is totally geodesic. This tangent space, which is the space of all symmetric matrices $V$ satisfying $\Tr(P^{-1}V)=0$, being of dimension $m-1$, we deduce that $\log_P R_1,\ldots,\log_P R_m$ are linearly dependent and thus that $P \in S_{R_1,\ldots ,R_m}$.
\vskip1ex
$\bullet$ $EBS(R_1,\ldots,R_m) \subseteq \mls{SPD}_1(n)$: let $P \in EBS(R_1,\ldots,R_m)$ so that there exists $\lambda \in \R^m$ for which $\sum_{i=1}^m \lambda_i =1$ and $\sum_{i=1}^m \lambda_i \log_P R_i = 0$, and let $d = \det(P)$. For each $i=1,\ldots,m$, we may decompose $\log_P R_i$ into its tangential and normal components to the submanifold $\mls{SPD}_d(n)$ which writes specifically: $\log_P R_i= \alpha_i P + V_i$ with $\alpha_i \in \R$ and $V_i$ satisfying $\Tr(P^{-1}V_i) = 0$. We then have:
\begin{equation*}
    \begin{aligned}
        R_i = \exp_P(\log_P R_i) &= P^{1/2} \exp(P^{-1/2}(\alpha_i P + V_i)P^{-1/2}) P^{1/2} \\
        &=P^{1/2} \exp(\alpha_i \text{Id} + P^{-1/2} V_i P^{-1/2}) P^{1/2} \\
        &=e^{\alpha_i} P^{1/2}\exp(P^{-1/2} V_i P^{-1/2}) P^{1/2}.
    \end{aligned}
\end{equation*}

\noindent Since $\det(R_i) = 1$, we thus get that
\begin{equation*}
      1=e^{n \alpha_i} \det(P)\exp(\Tr(P^{-1/2}V_iP^{-1/2})) = e^{n \alpha_i} \det(P) \exp(\Tr(P^{-1}V_i)) =  e^{n \alpha_i} d.
\end{equation*}

\noindent We deduce that $\alpha_i = \alpha = -\frac{1}{n} \log(d)$ and so, for all $i=1,\ldots,m$, $\log_P R_i = \alpha P + V_i$. Therefore, $\alpha \left(\sum_{i=1}^{m} \lambda_i \right) P + \sum_{i=1}^{m} \lambda_i V_i= \alpha P + \sum_{i=1}^{m} \lambda_i V_i = O$. Since $P$ is orthogonal to $\sum_{i=1}^{m} \lambda_i V_i$, it follows that necessarily $\alpha=0$, in other words $d=1$ and $P \in \mls{SPD}_1(n)$.
\vskip1ex
$\bullet$ We prove the last statement of the theorem by contraposition. Assume first that $P \in \mls{SPD}_1(n) \backslash EBS(R_1,\ldots,R_m)$, then one deduces as above that $\log_P R_1,\ldots,\log_P R_m$ are linearly dependent i.e.
$\sum_{i=1}^{m} \lambda_i \log_P R_i = 0$, with in addition $\sum_{i=1}^{m} \lambda_i = 0$. This means that $\{\log_P R_1,\ldots,\log_P R_m\}$ is a set of affinely dependent vectors in $T_P \mls{SPD}_1(n)$ and thus belongs to some affine subspace of dimension at most $m-2$ in $\mls{Symm}(n)$ which leads to the conclusion in this case.

On the other hand, if $P \in S_{R_1,\ldots ,R_m} \backslash \mls{SPD}_1(n)$, we have $\sum_{i=1}^m \lambda_i \log_P R_i = 0$ for some $\lambda \in \R^m\backslash\{0\}$. Using the same decomposition as previously, we write $\log_P R_i = \alpha_i P + V_i$ with $\Tr(P^{-1}V_i)=0$. Once again, we can show as above that for all $i$, $\alpha_i = \alpha =-\log(d)/n$, where $d\neq 1$ is the determinant of $P$, from which it follows that:
\begin{equation}
\label{eq:proof_th5}
    \sum_{i=1}^m \lambda_i \log_P R_i = \alpha \left(\sum_{i=1}^{m} \lambda_i \right) P + \sum_{i=1}^{m} \lambda_i V_i= 0.
\end{equation}

\noindent Now, since $P$ is orthogonal (with respect to $\langle\,\cdot\,, \,\cdot\, \rangle_P$) to $\sum_{i=1}^{m} \lambda_i V_i$ and $\alpha \neq 0$, this implies that $\sum_{i=1}^{m} \lambda_i =0$ and $\sum_{i=1}^{m} \lambda_i V_i = 0$. Let us now introduce $\Pi(P)$ the projection of $P$ onto $\mls{SPD}_1(n)$, which by \Cref{prop:tot_geod_SPD} is $\Pi(P) = d^{-1/n} P$. Then we have for each $i=1,\ldots,m$:
\begin{equation*}
    \begin{aligned}
    \log_{\Pi(P)} R_i &= \Pi(P)^{1/2} \log(\Pi(P)^{-1/2} R_i \Pi(P)^{-1/2}) \Pi(P)^{1/2} \\
    &=d^{-1/n} P^{1/2} \log(d^{1/n} P^{-1/2} R_i P^{-1/2}) P^{1/2} \\
    &=d^{-1/n} \log(d^{1/n} \Id) + d^{-1/n} P^{1/2} \log(P^{-1/2} R_i P^{-1/2}) P^{1/2} \\
    &=d^{-1/n} \frac{\log(d)}{n} \Id + d^{-1/n} \log_P R_i.
   \end{aligned}
\end{equation*}

\noindent Now using the fact that $\sum_{i=1}^{m} \lambda_i =0$ and \eqref{eq:proof_th5}, we get:
\begin{equation*}
   \sum_{i=1}^m \lambda_i \log_{\Pi(P)} R_i = d^{-1/n} \sum_{i=1}^m \log_P R_i = 0.
\end{equation*}

This implies that $\log_{\Pi(P)} R_1,\ldots,\log_{\Pi(P)} R_m$ are affinely dependent in $T_P \mls{SPD}_1(n)$ and thus, as in the previous case, that the reference points all belong to $\exp_P(H)$ for some affine subspace $H$ of dimension at most $m-2$.

\end{appendix}


\begin{thebibliography}{10}

\bibitem{afsari2011riemannian}
B.~Afsari.
\newblock {Riemannian $L^p$ center of mass: existence, uniqueness, and
  convexity}.
\newblock {\em Proceedings of the American Mathematical Society},
  139(2):655--673, 2011.

\bibitem{afsari2013convergence}
B.~Afsari, R.~Tron, and R.~Vidal.
\newblock {On the convergence of gradient descent for finding the {R}iemannian
  center of mass}.
\newblock {\em SIAM Journal on Control and Optimization}, 51(3):2230--2260,
  2013.

\bibitem{arnaudon2012stochastic}
M.~Arnaudon, C.~Dombry, A.~Phan, and L.~Yang.
\newblock Stochastic algorithms for computing means of probability measures.
\newblock {\em Stochastic Processes and their Applications}, 122(4):1437--1455,
  2012.

\bibitem{arsigny2007geometric}
V.~Arsigny, P.~Fillard, X.~Pennec, and N.~Ayache.
\newblock Geometric means in a novel vector space structure on symmetric
  positive-definite matrices.
\newblock {\em SIAM Journal on Matrix Analysis and Applications},
  29(1):328--347, 2007.

\bibitem{bancroft2007algebraic}
S.~Bancroft.
\newblock An algebraic solution of the {GPS} equations.
\newblock {\em IEEE Transactions on Aerospace and Electronic Systems},
  (1):56--59, 2007.

\bibitem{bergam2025t}
N.~Bergam, S.~Snoeck, and N.~Verma.
\newblock {t-SNE Exaggerates Clusters, Provably}.
\newblock {\em arXiv preprint arXiv:2510.07746}, 2025.

\bibitem{bonnabel2013stochastic}
Silvere Bonnabel.
\newblock {Stochastic gradient descent on {R}iemannian manifolds}.
\newblock {\em IEEE Transactions on Automatic Control}, 58(9):2217--2229, 2013.

\bibitem{bourgain1985lipschitz}
J.~Bourgain.
\newblock On {L}ipschitz embedding of finite metric spaces in {H}ilbert space.
\newblock {\em Israel Journal of Mathematics}, 52:46--52, 1985.

\bibitem{caputo2005class}
B.~Caputo, E.~Hayman, and P.~Mallikarjuna.
\newblock Class-specific material categorisation.
\newblock In {\em Tenth IEEE International Conference on Computer Vision},
  volume~2, pages 1597--1604. IEEE, 2005.

\bibitem{cury2013template}
C.~Cury, J.~A. Glaunes, and O.~Colliot.
\newblock Template estimation for large database: a diffeomorphic iterative
  centroid method using currents.
\newblock In {\em International Conference on Geometric Science of
  Information}, pages 103--111. Springer, 2013.

\bibitem{cobre_ml_figshare}
C.~Dansereau.
\newblock Cobre (for machine learning), 2015.

\bibitem{do1992riemannian}
M.~P. Do~Carmo and J.~Flaherty~Francis.
\newblock {\em Riemannian geometry}, volume~2.
\newblock Springer, 1992.

\bibitem{dryden2009non}
I.~L. Dryden, A.~Koloydenko, and D.~Zhou.
\newblock Non-{E}uclidean statistics for covariance matrices, with applications
  to diffusion tensor imaging.
\newblock {\em The Annals of Applied Statistics}, pages 1102--1123, 2009.

\bibitem{hekmati2020}
R.~Hekmati, R.~Azencott, W.~Zhang, Z.~D Chu, and M.~J. Paldino.
\newblock Localization of epileptic seizure focus by computerized analysis of
  {fMRI} recordings.
\newblock {\em Brain Informatics}, 7(1):1--13, 2020.

\bibitem{helgason1979differential}
S.~Helgason.
\newblock {\em Differential geometry, {L}ie groups, and symmetric spaces},
  volume~80.
\newblock Academic Press, 1979.

\bibitem{hinton2002stochastic}
G.~E. Hinton and S.~Roweis.
\newblock Stochastic neighbor embedding.
\newblock {\em Advances in Neural Information Processing Systems}, 15, 2002.

\bibitem{ho2013recursive}
J.~Ho, G.~Cheng, H.~Salehian, and B.~Vemuri.
\newblock Recursive {K}archer expectation estimators and geometric law of large
  numbers.
\newblock In {\em Artificial Intelligence and Statistics}, pages 325--332.
  PMLR, 2013.

\bibitem{kuhn1955hungarian}
H.~W. Kuhn.
\newblock {The Hungarian method for the assignment problem}.
\newblock {\em Naval Research Logistics Quarterly}, 2(1-2):83--97, 1955.

\bibitem{lloyd1982}
S.~Lloyd.
\newblock Least squares quantization in {PCM}.
\newblock {\em IEEE Transactions on Information Theory}, 28(2):129--137, 1982.

\bibitem{macqueen1967some}
J.~MacQueen.
\newblock Some methods for classification and analysis of multivariate
  observations.
\newblock In {\em Proceedings of the Fifth Berkeley Symposium on Mathematical
  Statistics and Probability}, volume~5, pages 281--298. University of
  California press, 1967.

\bibitem{matousek2013lectures}
J.~Matousek.
\newblock {\em Lectures on discrete geometry}, volume 212.
\newblock Springer Science \& Business Media, 2013.

\bibitem{mead1992review}
A.~Mead.
\newblock Review of the development of multidimensional scaling methods.
\newblock {\em Journal of the Royal Statistical Society: Series D (The
  Statistician)}, 41(1):27--39, 1992.

\bibitem{mityagin2020zero}
B.~S. Mityagin.
\newblock The zero set of a real analytic function.
\newblock {\em Mathematical Notes}, 107:529--530, 2020.

\bibitem{pennec2018barycentric}
X.~Pennec.
\newblock {Barycentric Subspace Analysis on Manifolds}.
\newblock {\em The Annals of Statistics}, 46(6A):2711--2746, 2018.

\bibitem{pennec2020manifold}
X.~Pennec.
\newblock Manifold-valued image processing with {SPD} matrices.
\newblock In {\em Riemannian Geometric Statistics in Medical Image Analysis},
  pages 75--134. 2020.

\bibitem{pennec2019riemannian}
X.~Pennec, S.~Sommer, and T.~Fletcher.
\newblock {\em Riemannian geometric statistics in medical image analysis}.
\newblock Academic Press, 2019.

\bibitem{petersen2006riemannian}
P.~Petersen.
\newblock {\em Riemannian geometry}, volume 171.
\newblock Springer, 2006.

\bibitem{sacksteder1960hypersurfaces}
R.~Sacksteder.
\newblock On hypersurfaces with no negative sectional curvatures.
\newblock {\em American Journal of Mathematics}, 82(3):609--630, 1960.

\bibitem{shiga1984}
K.~Shiga.
\newblock Hadamard manifolds.
\newblock {\em Geometry of Geodesics and Related Topics}, 3:239--282, 1984.

\bibitem{skovgaard1984riemannian}
L.~T. Skovgaard.
\newblock {A {R}iemannian geometry of the multivariate normal model}.
\newblock {\em Scandinavian Journal of Statistics}, pages 211--223, 1984.

\bibitem{sturm2003probability}
K.-T. Sturm.
\newblock Probability measures on metric spaces of nonpositive curvature.
\newblock {\em Heat kernels and analysis on manifolds, graphs, and metric
  spaces}, 338:357, 2003.

\bibitem{tan2024intrinsic}
C.~Tan, H.~Zhao, and H.~Ding.
\newblock Intrinsic {k}-means clustering over homogeneous manifolds.
\newblock {\em Pattern Analysis and Applications}, 27(3):107, 2024.

\bibitem{tou2008gabor}
J.~Y. Tou, Y.~H. Tay, and P.~Y. Lau.
\newblock Gabor filters as feature images for covariance matrix on texture
  classification problem.
\newblock In {\em International Conference on Neural Information Processing},
  pages 745--751. Springer, 2008.

\bibitem{tumpach2024totally}
A.~B. Tumpach and G.~Larotonda.
\newblock Totally geodesic submanifolds in the manifold {SPD} of symmetric
  positive-definite real matrices.
\newblock {\em Information Geometry}, 7(Suppl 2):913--942, 2024.

\bibitem{you2021}
K.~You and H.-J. Park.
\newblock Re-visiting {R}iemannian geometry of symmetric positive definite
  matrices for the analysis of functional connectivity.
\newblock {\em NeuroImage}, 225:117464, 2021.

\end{thebibliography}
\end{document}